%% file: paper.tex
\documentclass{fairmeta}
\usepackage{amsmath}
\usepackage{enumerate} 
\usepackage{bm}
\usepackage{floatflt}
\usepackage{algpseudocode}
\usepackage{amsfonts}
\usepackage{amsthm}
\usepackage{newtxtt}
\usepackage{colortbl} 
\usepackage{cleveref}
\usepackage{diagbox} 
\usepackage[utf8]{inputenc}
\usepackage{textgreek}
\usepackage{colortbl}
\usepackage{nicematrix}
\usepackage{makecell}
\usepackage{float}
\usepackage{arydshln}
\usepackage[frozencache,cachedir=.]{minted}
\usepackage{subcaption} 
\usepackage{tcolorbox}
\usepackage{amssymb}
\usepackage{xspace}
\usepackage{wrapfig}
\usepackage{adjustbox}
\usepackage{tabularx}
\usepackage{booktabs}
\usepackage{mathtools}
\usepackage{wrapfig}
\usepackage{amssymb}
\usepackage{graphicx}

\usepackage[flushleft]{threeparttable} 
\usepackage{multirow}

\usepackage[ruled,vlined]{algorithm2e} 

\usepackage{silence}
\makeatletter
\patchcmd{\wrong@fontshape}{\@gobbletwo}{}{}{}
\makeatother
\newtheorem{theorem}{Theorem}[]
\newtheorem{lemma}[theorem]{Lemma}

\newtheorem{remark1}[theorem]{Remark}

\definecolor{upColor}{RGB}{17,138,21}
\definecolor{downColor}{RGB}{174,36,67}

\newcommand{\up}[1]{\textcolor{upColor}{#1}}
\newcommand{\down}[1]{\textcolor{downColor}{#1}}

\DeclareRobustCommand{\metaicon}[2]{%
  \makebox[1.35em][c]{\raisebox{-0.18em}{\includegraphics[height=#1]{#2}}}\hspace{0.35em}%
}
\renewcommand\project[1]{\metadata[\metaicon{1.05em}{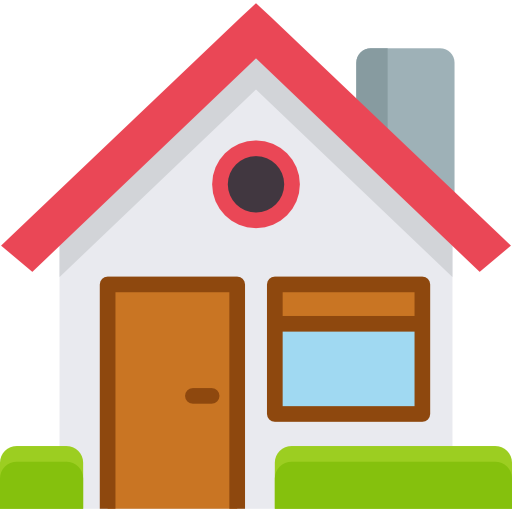}Project]{#1}}
\renewcommand\code[1]{\metadata[\metaicon{1.25em}{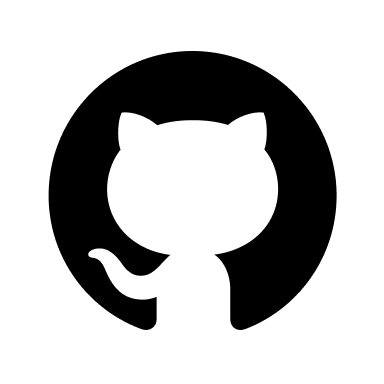}Code]{#1}}
\renewcommand\model[1]{\metadata[\metaicon{1.25em}{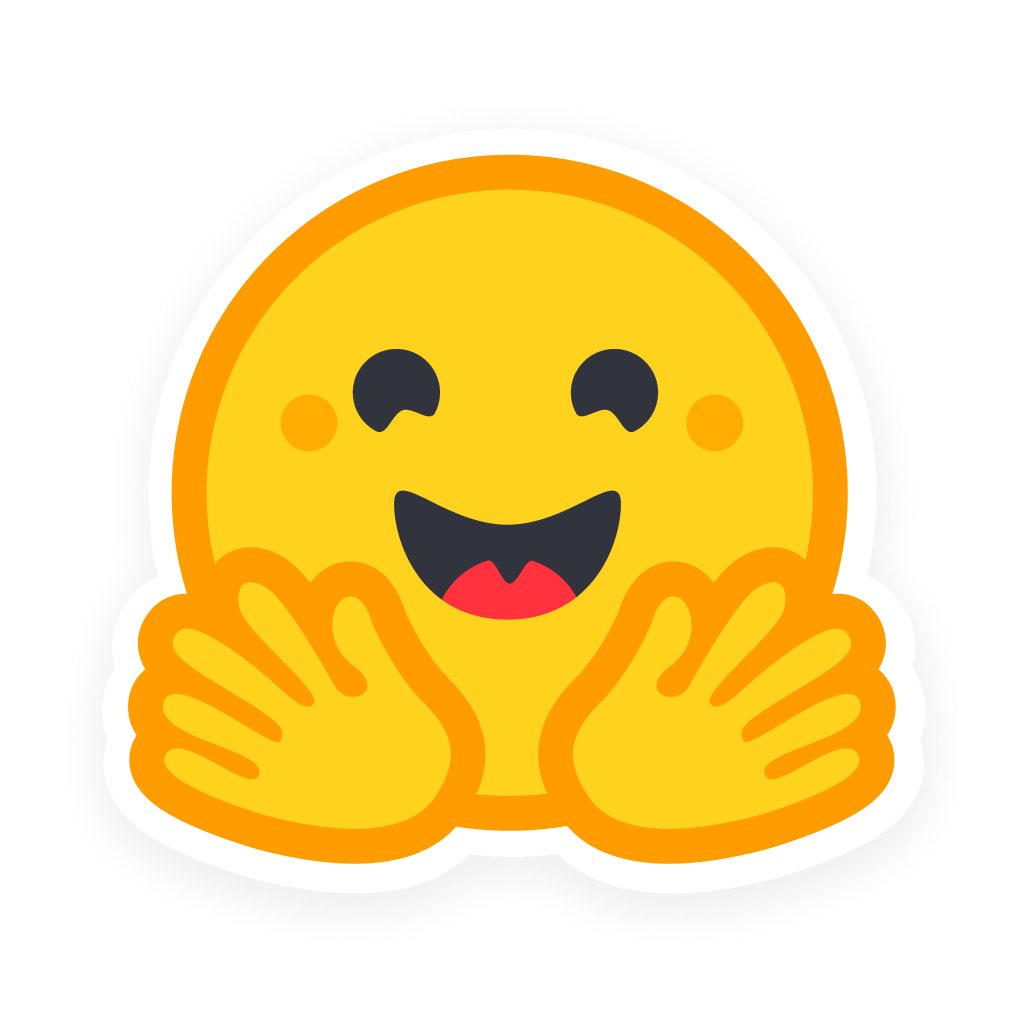}Model]{#1}}

\title{PRTS: A Primitive Reasoning and Tasking System via Contrastive Representations}
\author[1, 2\dagger, \ddagger]{\text{Yang Zhang}}   
\author[1, 3\ddagger]{\text{Jiangyuan Zhao}}
\author[1, 4]{\text{Chenyou Fan}}
\author[1]{\text{Fangzheng Yan}}
\author[1]{\text{Tian Li}}
\author[1]{\text{Haitong Tang}}
\author[1]{\text{Sen Fu}}
\author[1]{\text{Xuan'er Wu}}
\author[1]{\text{Qizhen Weng}}
\author[3]{\text{Weinan Zhang}}
\author[2]{\text{Xiu Li}}
\author[1]{\text{Chi Zhang}}
\author[1,*]{\text{Chenjia Bai}}
\author[1,*]{\text{Xuelong Li}}
\affiliation[1]{Institute of Artificial Intelligence (TeleAI), China Telecom}
\affiliation[2]{Tsinghua University}
\affiliation[3]{Shanghai Jiao Tong University}
\affiliation[4]{Fudan University}
\contribution{\textsuperscript{$\dagger$}Project Lead$\quad$\textsuperscript{$\ddagger$}Equal Contributions$\quad$\textsuperscript{*}Corresponding Authors}
\date{April, 2026}
\project{\url{https://rhodes-team-prts.github.io/}}
\code{\url{https://github.com/TeleHuman/PRTS}}
\model{\url{https://huggingface.co/TeleEmbodied/PRTS-4B}}
\metadata[Correspondence to]{Yang Zhang (\email{breezeyoung9470@gmail.com}) and Chenjia Bai (\email{baicj@chinatelecom.cn})\\ \\ }

\begin{document}
\input{sec/0_abs}

\maketitle

\input{sec/1_intro}
\input{sec/2_preliminaries}
\input{sec/3_method}

\input{sec/4_dataset}
\input{sec/5_related_work}
\input{sec/6_exp}
\input{sec/7_conclusion}

\clearpage

\bibliographystyle{plainnat}
\bibliography{paper}

\input{sec/X_suppl}

\end{document}

%% file: sec/0_abs.tex
\abstract{

Vision-Language-Action (VLA) models advance robotic control via strong visual-linguistic priors.
However, existing VLAs predominantly frame pretraining as supervised behavior cloning, overlooking the fundamental nature of robot learning as a goal-reaching process that requires understanding temporal task progress.
We present \textbf{PRTS} (\textbf{P}rimitive \textbf{R}easoning and \textbf{T}asking \textbf{S}ystem), a VLA foundation model that reformulates pretraining through Goal-Conditioned Reinforcement Learning.
By treating language instructions as goals and employing contrastive reinforcement learning, PRTS learns a unified embedding space where the inner product of state-action and goal embeddings approximates the log-discounted goal occupancy, the probability of reaching the language-specified goal from the current state-action, quantitatively assessing physical feasibility beyond static semantic matching.
PRTS draws this dense goal-reachability supervision directly from offline trajectories without reward annotations, and folds it into the VLM backbone via a role-aware causal mask, incurring negligible overhead over vanilla behavior cloning.
This paradigm endows the high-level reasoning system with intrinsic goal reachability awareness, bridging semantic reasoning and temporal task progress, and further benefits goal-conditioned action prediction.  
Pretrained on 167B tokens of diverse manipulation and embodied-reasoning data, PRTS reaches state-of-the-art performance on LIBERO, LIBERO-Pro, LIBERO-Plus, SimplerEnv, and a real-world suite of 14 complex tasks, with particularly substantial gains on long-horizon, contact-rich, and zero-shot novel-instruction settings, confirming that injecting goal-reachability awareness significantly improves both execution success and long-horizon planning of general-purpose robotic foundation policies.

}

%% file: sec/1_intro.tex
\section{Introduction}
\label{sec:intro}
Representation learning lies at the heart of robot learning, where the core challenge centers on extracting actionable knowledge from high-dimensional sensory inputs to enable goal-conditioned decision making. Vision-Language-Action (VLA) models have made remarkable strides by leveraging Vision-Language Models (VLMs) as powerful perceptual backbones, adopting a \emph{dual-system}---or \emph{Mixture-of-Transformers} (MoT)---architecture: System 2 (the VLM) handles high-level semantic understanding and task reasoning, while System 1 (the action expert) executes low-level continuous control \citep{black2024pi0, black2025pi05,bjorck2025gr00tn1}.
However, this paradigm overlooks a fundamental characteristic of embodied interaction: \emph{robotic trajectory is inherently a goal-reaching process over time}. This requires the representations inside a general-purpose robotic policy to not only encode static visual-linguistic semantic correlations, but also encapsulate the underlying \emph{temporal progress toward ultimate goal}---specifically, how close it is to achieving the goal at the current state.
While VLMs excel at encoding static semantic similarities between scenes and concepts \citep{robobrain,robovlm}, they intrinsically lack any notion of \emph{goal reachability}---a quantitative estimate of how likely a given state-action pair is to eventually achieve the language-specified objective.

Existing VLAs learn visuomotor priors almost exclusively through behavior cloning during pretraining \citep{kim2024openvla, zitkovich2023rt2, black2024pi0, black2025pi05, bjorck2025gr00tn1, zhai2025wallx, spiritv15, wu2026pragmatic}. Although some works \citep{black2025pi05, zhai2025wallx} additionally incorporate VQA data to preserve semantic reasoning capabilities, these approaches still focus on static understanding and immediate action prediction, without explicitly injecting temporal awareness of goal reachability into the pretrained VLA.
Consequently, System~2 operates with only qualitative semantic reasoning capabilities (determining ``\emph{what to do}'') but without the quantitative awareness of goal reachability needed to evaluate execution difficulty or to distinguish difficult-yet-reachable states from easy-yet-erroneous ones.

To address this deficiency, we reformulate VLA pretraining through the lens of \emph{Goal-Conditioned Reinforcement Learning}.
In particular, we adopt Contrastive Reinforcement Learning (CRL) \citep{eysenbach2022contrastiverl}, recently shown to scale to over 1000-layer deep networks \citep{1000-layer}, which learns paired state-action and goal representations whose similarity approximates the reachability of goals under given state-action pairs.
This formulation naturally equips the VLM backbone with temporal awareness of goal reachability, directly within its representation space.

In this paper, we present \textbf{PRTS} (\textbf{P}rimitive \textbf{R}easoning and \textbf{T}asking \textbf{S}ystem), a novel VLA foundation model that leverages \emph{Goal-Conditional RL} for representations learning in VLA pre-training phase.
Our approach introduces three key technical designs.
\textbf{(i)~Language-conditioned contrastive RL for representation learning.} 
Unlike prior VLAs \citep{black2024pi0, black2025pi05, bjorck2025gr00tn1, spiritv15, wu2026pragmatic} that pre-train the policy purely by behavior cloning, PRTS reformulates policy learning as a goal-conditioned representation learning problem with language instructions as goals.
Specifically, we construct positive pairs using state-action pairs from trajectories paired with their corresponding text instructions, and negative pairs using text goals from different tasks.
However, in contrast to standard CRL where future goals are sampled from a geometric distribution within the same episode, the goal is a language instruction shared across all timesteps in language-conditioned manipulation setup.
To bridge this gap, we transform the geometric distribution sampling into a temporal weighting scheme over state-action pairs.
This formulation drives the model to learn a representation space where the similarity between state-action embeddings and goal embeddings encodes the discounted goal occupancy measure, thereby capturing the temporal structure of goal-reaching behavior. 
\textbf{(ii)~Implicit dense goal-reachability supervision.}
The contrastive objective is built upon a goal-reaching reward assumption where the reward is defined as the transition probability of reaching the goal from the current timestep \citep{eysenbach2022contrastiverl}.
This enables the model to extract dense supervision and acquire a dense goal-conditioned action-value function directly from trajectory structure without manual reward annotations, quantifying the long-term utility of state-action pairs in purely offline pretraining.
\textbf{(iii)~Single-forward-pass design at negligible cost.}
Adding two contrastive heads to a VLM typically requires a second forward pass or a separately trained value network~\citep{pi-star-0.6}. 
PRTS instead appends two small token blocks, \texttt{<CRL\_action>} and \texttt{<CRL\_goal>}, to the standard input sequence and enforces information isolation through a role-aware causal mask integrated into a custom FlashAttention kernel~\citep{zadouri2026flashattention4}. 
The autoregressive action tokens retain standard causal attention, ensuring that the behavior-cloning loss remains identical with or without the CRL tokens. 
Meanwhile, the CRL tokens extract state-action embeddings and goal embeddings from isolated information streams within the same forward pass. 
As a result, end-to-end pretraining incurs negligible additional cost compared to vanilla behavior cloning on the same backbone.

Notably, PRTS unifies semantic reasoning, discrete action prediction, and goal-conditioned representation learning within a single VLM architecture, enjoying joint optimization and shared vision-language representations.
It equips the VLM with a quantitative notion of goal reachability that goes well beyond static semantic understanding.
Learning to estimate goal reachability requires the backbone to acquire task-relevant vision-language representations that encode the functional affordances required for task completion.

Built upon this architecture, we construct a large-scale pre-training dataset comprising diverse robotic trajectories with language annotations. 
We train the PRTS foundation model on over 167B tokens, employing infrastructure optimizations including a custom CuTe-FlashAttention kernel and sequence packing to achieve high-throughput training on 64 H100 GPUs for one week.
We subsequently fine-tune and evaluate the model on diverse downstream tasks: in simulation, we validate foundational manipulation capabilities on standard benchmarks including LIBERO \citep{liu2023libero}, LIBERO-Pro~\citep{zhou2025liberopro}, LIBERO-Plus~\citep{fei2025liberoplus}, and SimplerEnv (WidowX)~\citep{simpler};
in real-world deployment, we establish a comprehensive evaluation suite comprising 14 complex real-world manipulation tasks across a dual-arm RealMan platform and a single-arm Flexiv platform.
Experimental results demonstrate that PRTS achieves state-of-the-art performance in both simulation and real-world environments, exhibiting substantial advantages in long-horizon execution, zero-shot novel-instruction generalization, and recovery under human interventions.
These findings validate the efficacy of learning goal-reaching representations for enhancing planning robustness and execution success rates in VLA systems.

%% file: sec/2_preliminaries.tex
\section{Preliminaries}\label{sec:preliminaries}

\textbf{Goal-Conditioned Reinforcement Learning.~}
We consider the standard goal-conditioned reinforcement learning (GCRL) formulation, where an agent interacts with a Markov Decision Process (MDP) to reach a desired goal state $s_g \in \mathcal{S}$, where $\mathcal{S}$ is the state space. The environment is defined by state space $\mathcal{S}$, action space $\mathcal{A}$, transition dynamics $p(s' \mid s, a)$, and initial state distribution $p_0(s)$. 

Following \citet{eysenbach2022contrastiverl}, we focus on the \emph{goal-reaching reward} function, which quantifies the immediate probability of reaching the goal:
\begin{align}
    r_g(s_t, a_t) \triangleq (1-\gamma) \, p(s_{t+1} = s_g \mid s_t, a_t),
\end{align}
where $\gamma \in (0,1)$ is the discount factor. Unlike sparse 0/1 rewards, this formulation assigns $(1-\gamma)$ scaled by the transition probability to $s_g$, implicitly capturing the likelihood of goal achievement through the dynamics rather than 
hard thresholding.

Given a goal-conditioned policy $\pi(a \mid s, s_g)$, the $Q$-function represents the expected cumulative discounted return:
\begin{align}
    Q^{\pi}_{s_g}(s, a) \triangleq \mathbb{E}_{\pi} \left[ \sum_{t'=t}^{\infty} \gamma^{t'-t} r_g(s_{t'}, a_{t'}) \,\middle|\, s_t=s, a_t=a \right].
\end{align}

\textbf{Discounted State Occupancy Measure.~}
To establish the connection between value estimation and probability, we introduce the \emph{discounted state occupancy measure} \citep{ sutton2018reinforcement}. For a policy $\pi$ conditioned on goal $s_g$, the occupancy measure of state $s$ is defined as:
\begin{align}
    p^{\pi(\cdot \mid \cdot, s_g)}(s_{t+} = s) \triangleq (1-\gamma) \sum_{t=0}^{\infty} \gamma^t p^{\pi(\cdot \mid \cdot, s_g)}_t(s_t = s),
\end{align}
where $p^{\pi}_t(s_t = s)$ denotes the probability density of visiting state $s$ after executing $\pi$ for exactly $t$ steps starting from $p_0(s)$. This measure admits an intuitive sampling procedure: one first samples a time offset $t \sim \text{Geom}(1-\gamma)$ from a geometric distribution, then rolls out policy $\pi$ for $t$ steps. The resulting state follows the occupancy distribution $p^{\pi}(s_{t+})$. This property forms the theoretical foundation for constructing positive samples in standard contrastive RL.

\textbf{Q-Function as Occupancy Probability.~} 
A critical insight from goal-conditional RL is that the $Q$-function for goal-reaching rewards corresponds exactly to the conditional occupancy probability:

\begin{lemma}[\citet{eysenbach2022contrastiverl}]
    For the goal-reaching reward $r_g$, the Q-function equals the probability of reaching goal $s_g$ under the discounted state occupancy measure starting from $(s, a)$:
    \begin{align}
        Q^{\pi}_{s_g}(s, a) = p^{\pi(\cdot \mid \cdot, s_g)}(s_{t+} = s_g \mid s, a).
    \end{align}
\label{lemma:goal_reachability_as_q}
\end{lemma}
This equivalence transforms the problem of value estimation into one of probability estimation: estimating how likely the agent is to visit $s_g$ in the future when following $\pi$ from $(s, a)$.

\textbf{Contrastive RL Framework.~}
Building upon this equivalence, Contrastive RL estimates the $Q$-function via contrastive representation learning, avoiding the bootstrapping errors inherent in Temporal-Difference (TD) learning. The critic $f(s, a, s_g)$ measures the correlation between a state-action pair $(s, a)$ and a future goal $s_g$. Following standard practice, we parameterize it as the inner product between encoded representations:
\begin{align}
f(s, a, s_g) = \phi(s, a)^{\top} \psi(s_g),
\end{align}
where $\phi: \mathcal{S} \times \mathcal{A} \rightarrow \mathbb{R}^d$ is the state-action encoder and $\psi: \mathcal{S} \rightarrow \mathbb{R}^d$ is the goal encoder. 

Contrastive RL objective learns $f$ by discriminating between 
\emph{achievable future states} (positive samples) and 
\emph{random states} (negative samples), both denoted as $s_g$ 
but drawn from different distributions:
\begin{itemize}
    \item \textbf{Positive:} A goal state sampled from the discounted 
    occupancy measure, $s_g^+ \sim p^{\pi(\cdot \mid \cdot)}(s_{t+} \mid s, a)$, 
    representing a state actually reachable by $\pi$ from $(s, a)$;
    \item \textbf{Negative:} A goal state sampled from the marginal 
    data distribution, $s_g^- \sim p(s_g)$, representing a random 
    state unrelated to the current policy.
\end{itemize}

The critic is trained via the CRL objective based on the infoNCE objective~\citep{info-nce}:
\begin{align}\label{eq:l-crl}
    \mathcal{L}_{\rm crl}=\max_{f} \,\, \mathbb{E}_{\substack{(s,a) \sim p(s,a), \\ 
    s_g^+ \sim p^{\pi}(s_{t+} \mid s,a), \\ 
    s_g^- \sim p(s_g)}} 
    \Big[ \log \frac{\exp (f(s, a, s_g^+)) }{ \exp(f(s,a,s_g^+)) + \sum \exp (f(s,a,s_g^-)) } \Big],
\end{align}
The optimal critic satisfies: $f^*(s, a, s_g) = \log \left( \frac{p^{\pi}(s_{t+}=s_g \mid s,a)}{p(s_g)} \right) + c(s,a)$, where $c(s,a)$ is a function independent of $s_g$. 
Leveraging Lemma~\ref{lemma:goal_reachability_as_q} where $Q^{\pi}_{s_g}(s,a) = p^{\pi}(s_{t+}=s_g \mid s,a)$, 
this yields the direct correspondence:
\begin{align}
    f^*(s, a, s_g) = \log Q^{\pi}_{s_g}(s,a) - \log p(s_g) + c(s,a),
\end{align}
implying $\exp(f^*(s, a, s_g)) \propto Q^{\pi}_{s_g}(s,a)$ up to a state-action-dependent normalization. Thus, the critic provides 
a valid estimate of the relative Q-values for policy evaluation.

%% file: sec/3_method.tex
\section{Method}

\subsection{Problem Setup}
\label{sec:problem}

We consider a language-conditioned robotic manipulation task with visual observations under the imitation learning paradigm. 
Let $\mathcal{D}_\text{pre} = \{ \tau_1, \tau_2, \ldots, \tau_n \}$ denote a large-scale pre-training dataset comprising $n$ demonstration trajectories collected from diverse robotic tasks. 
Each trajectory $\tau = (s_{1:T}, l, a_{1:T})$ consists of a natural language instruction $l$ describing the task objective, a sequence of observations $s_{1:T}$, and corresponding actions $a_{1:T}$. 
At timestep $t$, the observation $s_t = (\mathbf{I}_t^1, \ldots, \mathbf{I}_t^V, \mathbf{q}_t)$ comprises $V$ RGB camera images and the robot's proprioceptive state $\mathbf{q}_t \in \mathbb{R}^{d_q}$. The action $a_t \in \mathbb{R}^{d_a}$ represents the low-level control commands.

The central objective of VLA pre-training is to learn representations that encode not only visual-linguistic semantics but also the \emph{temporal structure of goal-reaching}---quantifying how likely a given observation-action pair is to eventually achieve the language-specified objective, i.e., \emph{goal reachability}.
While recent works such as $\pi_{0.6}^\star$~\citep{pi-star-0.6} and VLAC~\citep{VLAC} have explored value-augmented VLA training, they share fundamental bottlenecks: 
(\romannumeral1)~reliance on explicit reward or progress annotations, e.g., episode success/failure labels for $\pi^*_{0.6}$ and curated pair-wise progress labels for VLAC;
(\romannumeral2)~a separately trained value network that incurs substantial additional training cost or requires a multi-stage pipeline;
and (\romannumeral3)~the categorical targets are either discretized scalar returns or hand-engineered progress fractions, none of which directly encode the underlying goal-reaching structure of the task.

In contrast, PRTS frames goal reachability estimation as \emph{contrastive RL}~\citep{eysenbach2022contrastiverl}: the categorical structure arises naturally from the multi-task training batch---each language goal serves as a class, and the model learns to classify which goal a given $(s, a)$ is progressing toward.
This formulation eliminates the need for reward labels, return discretization, pairwise data curation, and auxiliary value networks.
Instead, the resulting goal-conditioned action-value function $Q^\pi_l(s, a)$ emerges directly from the representation geometry of the shared VLM backbone within a single forward pass.
Furthermore, recent advances demonstrate that contrastive objectives exhibit favorable scaling properties in architectures with thousands of layers \citep{1000-layer}, holding promise for the integration with large-scale VLA foundation models.
Moreover, their cross-entropy formulation aligns naturally with the discrete token prediction paradigm of VLMs, enabling unified pre-training where goal-conditioned representation learning and language modeling objectives are optimized jointly.

\subsection{Language-Conditioned Contrastive Reinforcement Learning}
\label{sec:language_crl}

Adapting standard contrastive RL to the vision-language-action setting induces a shift in how task goals are parameterized.
Standard CRL operates in state-based goal settings: for each state-action pair $(s_t, a_t)$ sampled from a trajectory, a goal state $s_g$ is drawn via geometric sampling $\Delta t \sim \text{Geom}(1-\gamma)$ from future timesteps.
Since $s_{t+\Delta t}$ is almost surely unique within a batch, this yields a \emph{single-positive} contrastive pair in which every other sample's future state serves as a distinct negative goal.
In PRTS, the goal is instead defined as the \emph{language instruction} $l$, which specifies the task-completion condition and is shared across all timesteps of a trajectory. This shared-goal structure breaks the single-positive assumption and demands a reformulation of CRL.

\textbf{Multi-Positive Problem.} Since the goal $l$ is identical across all timesteps of every trajectory belonging to the same task, a mini-batch drawn from the dataset typically contains many state-action pairs that share a single language goal. Formally, consider a mini-batch $\mathcal{B} = \{(s_j, a_j, l_j, \tau_j)\}_{j=1}^{B}$ of size $B$ sampled from the pre-training dataset $\mathcal{D}_{\rm pre}$, where each tuple contains:
\begin{itemize}
    \item $(s_j, a_j)$: the state-action pair at timestep $t_j$,
    \item $l_j$: the language instruction (goal) for the trajectory,
    \item $\tau_j$: the task identifier (indicating which task the sample belongs to).
\end{itemize}

For a given anchor sample $i \in \{1, \ldots, B\}$, we define the \emph{positive set} as the indices of all samples in the batch belonging to the same task:
\begin{align}
    \mathcal{S}(i) \triangleq \{ j \in \{1, \ldots, B\} : \tau_j = \tau_i \}.
\end{align}
That is, $\mathcal{S}(i)$ collects all batch indices that share the anchor's task, including $i$ itself. This structural property prevents direct application of standard CRL's single-positive objective (Eq.~\eqref{eq:l-crl}), necessitating a reformulation that handles multiple positives while still preserving the temporal structure of goal-reaching.

\textbf{Temporal Weighting as Geometric Sampling.} 
In standard CRL, geometric sampling $\Delta t \sim \text{Geom}(1-\gamma)$ makes the probability of sampling a future state $k$ steps ahead proportional to $\gamma^k$. 
Under this sampling, the optimal classifier satisfies
\begin{align}\label{eq:f-star}
    f^*(s, a, s_g) = \log \frac{p^{\pi}(s_{t+}=s_g \mid s,a)}{p(s_g)} + c(s,a),
\end{align}
where $p^{\pi}(s_{t+}=s_g \mid s,a) = (1-\gamma)\sum_{k=1}^{\infty} \gamma^{k-1} p(s_{t+k}=s_g \mid s_t, a_t)$ is the discounted goal occupancy measure and $c(s,a)$ is a goal-independent term that is absorbed into the learned bias of the network and leaves the relative temporal information intact. 

However, language goals cannot be sampled geometrically from a trajectory---they are fixed per task.
Fortunately, in expert demonstrations, the discounted probability of reaching the goal from timestep $t$ can be approximated by $\gamma^{T-t}$, i.e., the discount factor raised to the power of remaining timesteps. 
This observation motivates a temporal weighting scheme that emulates geometric sampling by assigning weights to multiple positive samples according to their temporal distance to the goal.
Formally, for each positive sample $j \in \mathcal{S}(i)$ at trajectory timestep $t_j$ (with length $T_j$) sharing the same language goal $l_i$, we define the temporal weight:
\begin{align}\label{eq:temporal-weight}
    q_{ij} = \frac{\gamma^{T_j - t_j}}{\sum_{j' \in \mathcal{S}(i)} \gamma^{T_{j'} - t_{j'}}}.
\end{align}
This weighting scheme assigns exponentially larger weights to states closer to task completion and mirrors the decay of the geometric distribution. 
As we prove below, optimizing with these soft targets yields representations whose inner product $\psi(l)^{\top}\phi(s,a)$ is proportional to the log-discounted occupancy, recovering the same temporal structure as standard CRL.




\textbf{Bidirectional Contrastive Objectives.} 
We optimize two complementary cross-entropy objectives that play distinct representational roles.

\textbf{(i) State-Action to Language ($s,a \to l$):} 
For each state-action anchor sampled from the batch, we compute:
\begin{align}\label{eq:sal-loss}
    \mathcal{L}^{\text{sa} \to l} = \mathbb{E}_{i \sim \mathcal{B}} \left[ - \gamma^{T_i - t_i} \log \frac{\exp(\phi_i^{\top} \psi_i)}{ \exp(\phi_i^{\top} \psi_i) +  \sum_{k \in \mathcal{B}, \, \tau_k \neq \tau_i}\exp(\phi_i^{\top} \psi_k)} \right],
\end{align}
where $\phi_i \triangleq \phi(s_i, a_i)$ and $\psi_i \triangleq \psi(l_i)$, and the denominator sums over all unique language goals in the batch (including negatives).
In this direction, each $(s_i, a_i)$ has exactly \emph{one} positive goal $l_i$ (its own task's instruction), reducing the formulation to the standard CRL problem.
Thus, to present fixed goals under the geometric sampling distribution, we weight each state-action pair by $\gamma^{T_i - t_i}$.

However, sample-wise weight $\gamma^{T-t}$ alone only rescales the loss magnitude without altering the gradient direction, pushing all states from the same task toward the same similarity with $l_i$ and failing to encode the temporal progression within the trajectory.
Consequently, this objective primarily promotes \emph{task-level discrimination}: it aligns state-action pairs with their corresponding language goals, while remaining largely agnostic to temporal distance to task completion.

\textbf{(ii) Language to State-Action ($l \to s,a$):} 
In contrast, the temporal signal is introduced in the reverse direction.
For a language anchor $l_i$ with positive set $\mathcal{S}(i) = \{j \in \mathcal{B} : \tau_j = \tau_i\}$, we optimize:
\begin{align}\label{eq:lsa-loss}
    \mathcal{L}^{l \to \text{sa}} = \mathbb{E}_{i \sim \mathcal{B}} \left[ -\sum_{j \in \mathcal{S}(i)} q_{ij} \log \frac{\exp(\psi_i^{\top} \phi_j)}{\sum_{k \in \mathcal{B}} \exp(\psi_i^{\top} \phi_k)} \right], \quad \text{with} \quad q_{ij} = \frac{\gamma^{T_j - t_j}}{\sum_{j' \in \mathcal{S}(i)} \gamma^{T_{j'} - t_{j'}}}.
\end{align}
Unlike the $sa\to l$ direction, $l_i$ now has \emph{multiple} positives.
The soft targets $q_{ij}$ require the predicted probability that the state-action pair $j$ belongs to task $i$ to scale as $\gamma^{T_j - t_j}$.
This forces the representation inner product to satisfy $\psi_i^{\top} \phi_j = (T_j - t_j)\log\gamma + C$ according to Eq.\eqref{eq:f-star}, thereby encoding the \emph{temporal distance to task completion} within the representation space.

The combination of \textbf{(i)} and \textbf{(ii)} ensures that the learned representations capture both \emph{task identity} (via $s,a\to l$) and \emph{task progress} (via $l\to s,a$). 
While $s,a\to l$ prevents interference between different tasks, $l\to s,a$ arranges states along a trajectory in order of their value (proximity to goal), creating a structured representation space where the inner product $\psi(l)^{\top}\phi(s,a)$ serves as a valid estimate of the log-discounted occupancy measure.

\textbf{Theoretical Equivalence to CRL.}
We now establish that the $l\to s,a$ objective yields representations equivalent to
standard CRL’s discounted occupancy estimation.

\setcounter{theorem}{0}
\begin{theorem}[Temporal Weighting Implements Geometric Sampling]
\label{thm:equivalence}
Let $\pi^*$ be the deterministic expert policy generating the demonstrations. For any state-action pair $(s, a)$ in the dataset belonging to task with goal $l$, the optimal representations $(\phi^*, \psi^*)$ minimizing $\mathcal{L}^{l \to \text{sa}}$ satisfy:
\begin{align}
    \psi^*(l)^\top \phi^*(s, a) = \log Q^{\pi^*}_l(s, a) + C(l),
\end{align}
where $Q^{\pi^*}_l(s, a) = p^{\pi^*}_{\gamma}(s_{t+}=s_g \mid s, a)$ is the discounted state occupancy measure (i.e., the probability of reaching goal $s_g$ under $\pi^*$ starting from $(s, a)$), and $C(l)$ is a function depending only on the goal $l$.
\end{theorem}

\begin{proof}
The proof is deferred to Appendix~\ref{app:proof}.     
\end{proof}

Theorem~\ref{thm:equivalence} reveals that while the weights $q_{ij}$ are computed using temporal distances $(T-t)$, the learned representations do not merely encode `how many steps remain.' Instead, they encode the \emph{logarithm of the discounted occupancy measure} $\log Q^{\pi}_l(s,a)$, which is the fundamental quantity in goal-conditioned RL. Finally, we emphasize that although the derivation of the whole section is presented using single-step actions, the formulation naturally extends to action chunks used in VLAs, without changing the objectives or theoretical results.



\subsection{PRTS Architecture}
\label{sec:prts_arch}

\begin{figure}[t]
    \centering
    \includegraphics[width=1.0\linewidth]{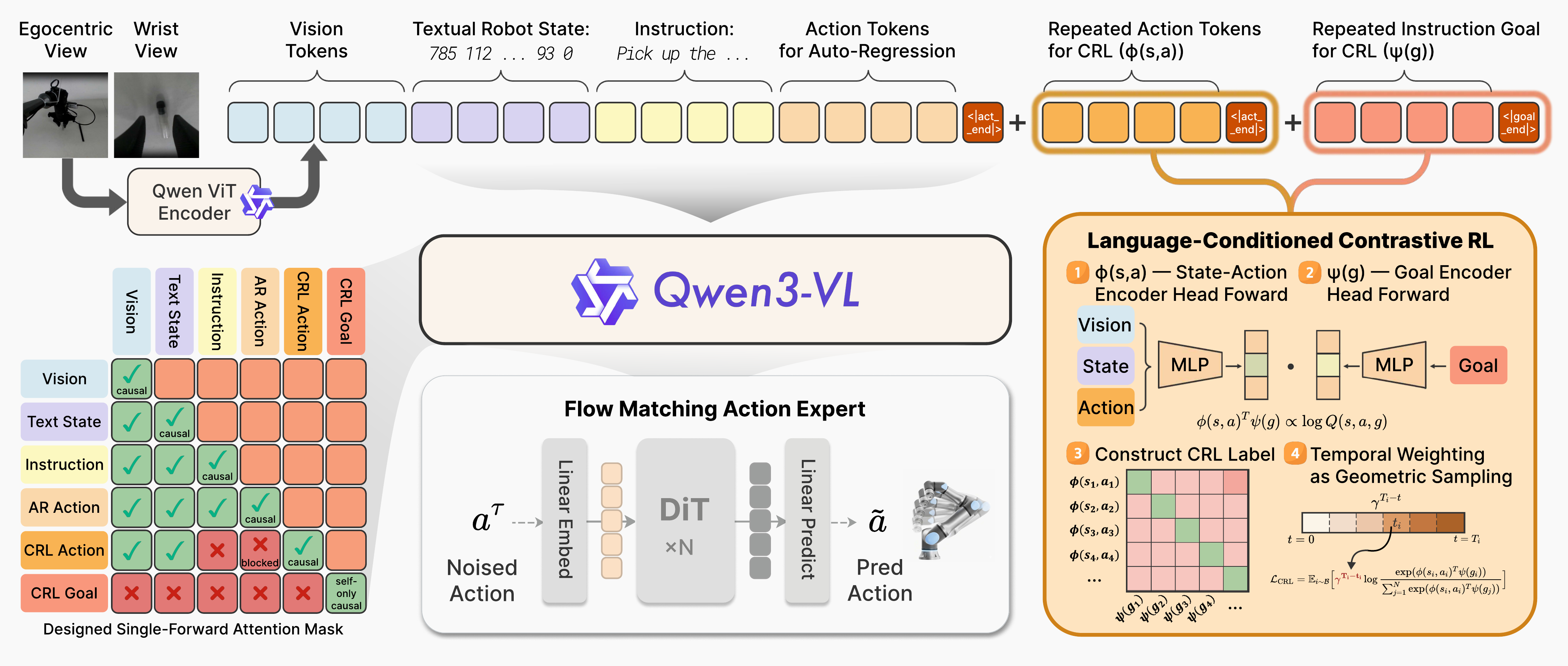}
    \caption{\textbf{Overview of PRTS.}
    PRTS is built on Qwen3-VL-4B-Instruct and involves two training stages.
    (\romannumeral1) In the \textbf{pre-training} stage, the backbone consumes a unified token sequence of multi-view visual tokens (Qwen ViT), textual proprioceptive-state tokens, language-instruction tokens, and FAST-tokenized AR action tokens, followed by two auxiliary blocks \texttt{<CRL\_action>} and \texttt{<CRL\_goal>}, in which the last learnable tokens \texttt{<|action\_end|>} and \texttt{<|goal\_end|>} are passed through lightweight MLP heads (right) to produce the state-action representation $\phi(s_t,\mathbf{a}_t)$ and the goal representation $\psi(l)$.
    A \emph{role-aware causal mask} (bottom-left) prevents information leakage across the five token roles: \texttt{<CRL\_action>} tokens attend only to vision and proprioception; \texttt{<CRL\_goal>} tokens attend only within their own block; AR action tokens keep standard causal attention.
    All three outputs---discrete-action logits, $\phi(s_t,\mathbf{a}_t)$, and $\psi(l)$---are produced in a single forward pass and trained jointly under the bidirectional contrastive (InfoNCE) losses with temporal weighting and the next-token cross-entropy loss, so that $\phi(s_t,\mathbf{a}_t)^{\top}\psi(l) \approx \log Q^{\pi}_l(s_t,\mathbf{a}_t)$.
    The action expert is absent at this stage.
    (\romannumeral2) In the \textbf{post-training} stage, a freshly initialized DiT-based flow-matching action head is attached and fine-tuned via $\mathcal{L}_{\text{FM}}$ to produce continuous action chunks.}
    \label{fig:architecture}
\end{figure}

We instantiate PRTS upon the Qwen3-VL backbone \citep{Qwen3-vl}, extending it to support unified pre-training for both discrete action prediction and contrastive representation learning. As illustrated in Fig.~\ref{fig:architecture}, the input sequence comprises visual tokens (encoded via a Qwen ViT Encoder from egocentric and wrist views $\mathbf{I}^1_t, \dots, \mathbf{I}^V_t$), robot proprioception tokens for $\mathbf{q}_t$, language instruction tokens for $l$, and auto-regressive (AR) action tokens obtained by FAST~\citep{pertsch2025fast} tokenization of action chunk $\mathbf{a}_t = a_{t:t+H}$.
Here, we represent proprioceptive states as text after discretization like \citep{black2025pi05, driess2025ki}.
Crucially, we append two auxiliary token blocks at the tail of the sequence: \texttt{<CRL\_action>} and \texttt{<CRL\_goal>}. 
The \texttt{<CRL\_action>} block consists of repeated action tokens followed by a learnable token \texttt{<|action\_end|>}, whose hidden state encodes $\phi(s_t,\mathbf{a}_t)$. 
Similarly, the \texttt{<CRL\_goal>} block consists of repeated instruction tokens ended with a learnable token \texttt{<|goal\_end|>}, whose hidden state encodes $\psi(l)$. 
These auxiliary tokens extract representations for contrastive learning within the same forward pass used for action generation.

\textbf{Role-Aware Causal Mask for Disentangled Representations.}
It is non-trivial to ensure that $\phi(s_t,\mathbf{a}_t)$ and $\psi(l)$ correctly encode the desired modalities without information leakage while preserving the efficiency of standard next-token prediction training.
To address this, we introduce a \emph{role-aware} causal attention mask, depicted in the bottom-left of Fig.~\ref{fig:architecture}, which assigns each token one of five roles: vision/state, instruction, AR action, \texttt{CRL\_action}, and \texttt{CRL\_goal}.
Based on these roles, we enforce three structured attention constraints:

(\romannumeral1)~{AR action tokens} use standard causal attention over all preceding visual, proprioception, and instruction tokens.
This preserves the original autoregressive action token prediction, ensuring that the behavior cloning loss on FAST-tokenized actions remains identical to a non-CRL run.

(\romannumeral2)~\texttt{CRL\_action} tokens $\phi(s_t,\mathbf{a}_t)$ are masked to attend \emph{only} to vision tokens, proprioception tokens, and the CRL action tokens themselves. They are blocked from attending to language instructions and AR action tokens. This ensures that $\phi(s,a)$ extracts a pure state-action representation independent of the language goal, which is essential for the validity of the contrastive similarity $\phi^\top\psi$.

(\romannumeral3)~\texttt{CRL\_goal} tokens $\psi(l)$ employ self-only causal attention restricted to the \texttt{<CRL\_goal>} token block, with all other tokens masked out. This forces $\psi(l)$ to be a compact, task-specific encoding of the language instruction alone, serving as the goal anchor in the contrastive objective.

This design enables computing both action predictions and contrastive representations ${\phi}(s_t, \mathbf{a}_t), {\psi}(l)$ within a single forward pass, preserving training efficiency.

\textbf{Practical Implementations.}
We implement the structurally sparse role-aware mask on top of the latest FlashAttention~\citep{zadouri2026flashattention4} combined with a custom CuTe kernel, which encodes the five-role rules at the block level so that only the non-masked blocks are computed. This brings the joint forward pass to the same throughput as a pure-causal next-token-prediction behavior cloning run with FlashAttention-2/3 as backend~\citep{dao2023flashattention2, shah2024flashattention3}.
The design further remains compatible with sequence packing trick which we also involve to accelerate our large-scale pre-training.
PRTS therefore produces the discrete-action logits, $\{\phi_i\}$, and $\{\psi_i\}$ in a single forward pass whose per-step cost closely matches that of a pure-BC run on the same Qwen3-VL backbone (quantified in Sec.~\ref{sec:exp}), a property we rely on to scale contrastive-RL pre-training to a over 167B token corpus.

\textbf{Action Expert.}
For real-time continuous control, we adopt a flow-matching action expert~\citep{lipman2023flow, liu2023flow, lipman2024flowguide} based on a Diffusion Transformer (DiT)~\citep{peebles2023dit}, following recent VLA paradigms~\citep{black2024pi0, black2025pi05}. 
It outputs the action chunk $\tilde{\mathbf{a}}_t$ with a chunk horizon $H$ over 5 denoising steps, conditioned on the VLM backbone hidden states, enabling high-frequency continuous action generation while the backbone focuses on high-level semantic reasoning.
The action expert $f_\theta$ is trained via conditional flow matching~\citep{lipman2024flowguide}:
\begin{align}\label{eq:fm-loss}
    \mathcal{L}_{\text{FM}} = 
    \mathbb{E}_{\tau \sim [0,1],\,\epsilon \sim \mathcal{N}(0, I)}
    \left\|
    f_\theta\!\left(\mathbf{a}^{\tau}_{t},\, s_t,\, l\right)
    - (\mathbf{a}_{t} - \epsilon)
    \right\|_2^2,
\end{align}
where $\mathbf{a}^{\tau}_{t} = \tau\, \mathbf{a}_{t} + (1-\tau)\,\epsilon$ is the noised action chunk, and $\tau\in [0,1]$ denotes the flow-matching time index.

\textbf{Pretraining and Post-training Objectives.}
PRTS adopts a two-stage training scheme that cleanly separates representation learning from continuous-action adaptation.
The VLM backbone is shaped during pre-training, and the flow-matching action expert is attached and adapted during post-training. Gradients flow through the VLM backbone in both stages.

In the \emph{pre-training} stage,
the backbone is trained jointly on the bidirectional contrastive losses, and the standard next-token prediction cross-entropy loss, which encompasses the BC loss on AR action tokens and the auxiliary cross-entropy losses for robot-QA and sub-task prediction.
The action expert is not included at this stage.
At each iteration, a mini-batch $\mathcal{B}$ yields the temporal weights $q_{ij}$ of Eq.~\eqref{eq:temporal-weight}, and the representations $\phi(s_{t_j}, \mathbf{a}_{t_j})$ and $\psi(l_i)$ are respectively produced by lightweight MLP heads applied to the final hidden states in the \texttt{<CRL\_action>} and \texttt{<CRL\_goal>} blocks, followed by $\ell_2$-normalization and scaling by a learnable temperature $\tau_{\text{crl}}$. The resulting dot products $\phi_j^{\top}\psi_i$ are then used to compute the bidirectional contrastive objectives: the state-action to language loss $\mathcal{L}^{\text{sa} \to l}$ (Eq.~\eqref{eq:sal-loss}) and the language to state-action loss $\mathcal{L}^{l \to \text{sa}}$ (Eq.~\eqref{eq:lsa-loss}).
As demonstrated by \citet{oord2018cpc}, the performance of the InfoNCE objective relies on the size of the negative pool.
To maximize this pool without exceeding memory constraints, we employ a cross-device contrastive strategy akin to CLIP~\citep{radford2021clip}. Specifically, we shard the similarity computation across GPUs, allowing each device to aggregate negative representations globally while only computing the gradient for its local batch.
The pre-training objective is
\begin{align}\label{eq:loss-pre}
    \mathcal{L}_{\text{pre-train}} = \mathcal{L}_{\text{BC}} + \mathcal{L}_{\text{RobotQA}} + \mathcal{L}_{\text{Subtask}} + \lambda_{\text{crl}} \left( \mathcal{L}^{\text{sa} \to l} + \mathcal{L}^{l \to \text{sa}} \right),
\end{align}
where $\mathcal{L}_{\text{BC}}$ is the cross-entropy loss on AR action tokens and $\mathcal{L}_{\text{RobotQA}}, \mathcal{L}_{\text{Subtask}}$ are the auxiliary robot-QA and sub-task prediction losses. This unified objective enables the VLM backbone to learn goal-conditioned representations that simultaneously support (\romannumeral1)~semantic grounding, (\romannumeral2)~discrete and continuous action generation over action chunks, and (\romannumeral3)~dense goal-reachability estimation expressed via the inner product of the contrastive embeddings.

In the \emph{post-training} stage,
we then attach the flow-matching action expert and fine-tune on downstream demonstration data using Eq.~\eqref{eq:fm-loss},
\begin{align}\label{eq:loss-post}
    \mathcal{L}_{\text{post-train}} = \mathcal{L}_{\text{FM}}.
\end{align}
We drop the BC and auxiliary terms in this stage.
The representations after pre-training already encode task identity and goal reachability, so post-training can focus entirely on continuous-action regression.

\textbf{Benefits for Representation and Value Extraction.}
In Eq.~\eqref{eq:loss-pre}, while $\mathcal{L}_{\text{BC}}$ focuses on temporally local, low-level action accuracy, the CRL objectives encourage the model to learn a \emph{globally consistent} representation space whose geometry encodes the temporal structure of goal-reaching. 
Specifically, the contrastive signal aligns visual observations, robot states, and language instructions into a unified metric space in which distances correspond to task progress. 
This mitigates overfitting to spurious correlations in demonstration data and promotes generalization to long-horizon and fine-grained manipulation tasks.

Upon training completion, the model inherently provides a dense goal-conditioned action-value without requiring a separately trained value network. 
The learned representations satisfy:
\begin{align}
    Q^{\pi}_l(s, \mathbf{a}) \propto \exp\left( \phi(s, \mathbf{a})^\top \psi(l) \right),
\end{align}
or equivalently, $\log Q^{\pi}_l(s, \mathbf{a}) \approx \phi(s, \mathbf{a})^\top \psi(l) + \text{const}$.
As established in Theorem~\ref{thm:equivalence}, this inner product estimates the log-probability of reaching the language goal under the discounted state-occupancy measure, yielding a scalar signal that quantifies progress toward satisfying the instruction.
Notably, this dense $Q^{\pi}_l(s,\mathbf{a})$ emerges as a by-product of the same backbone used for action generation, suggesting that value awareness is implicitly integrated into the policy during pre-training and embedded directly into the learned representations.






%% file: sec/4_dataset.tex
\section{Datasets}\label{sec:dataset_overview}

\begin{figure}[t]
    \centering
    \includegraphics[width=1.0\linewidth]{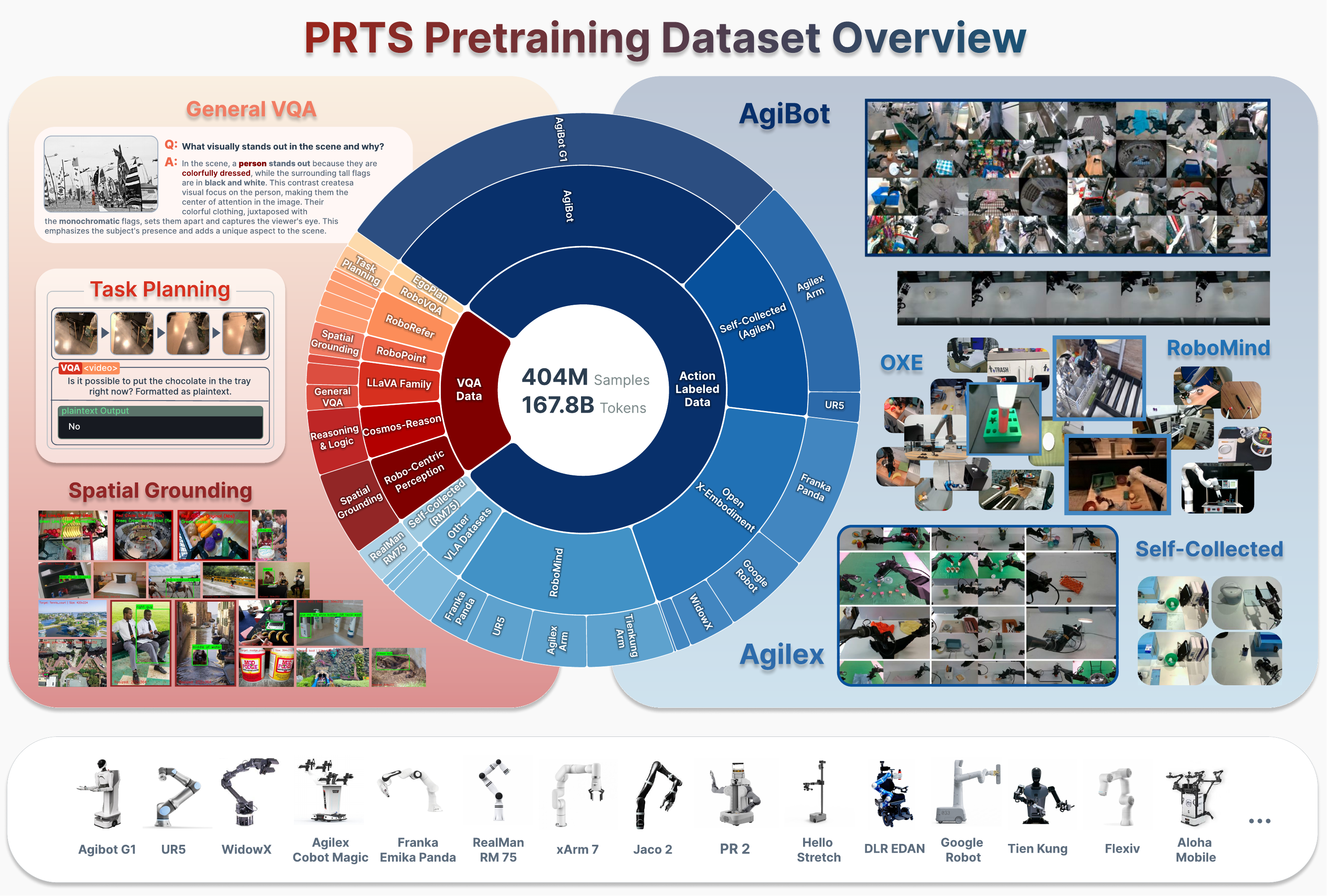}
    \caption{\textbf{Pre-training dataset overview of PRTS.} Our pre-training dataset is partitioned into two primary domains: \emph{Action-Labeled Data} provides physical execution priors for cross-embodiment manipulation, while \emph{Visual-Reasoning Data} provides complementary supervision for spatial grounding, affordance perception, and task planning.
    The area of each segment is scaled proportionally to the square root of its total tokens to highlight the rich, structural diversity of our data sources.}
    \label{fig:dataset_overview}
\end{figure}

To build a value-aware foundation VLA capable of both high-level reasoning and low-level control, we construct a large-scale multi-modal pre-training dataset of \textbf{404\,M samples} totaling \textbf{167.8\,B tokens}. As illustrated in Fig.~\ref{fig:dataset_overview}, the dataset is partitioned into two complementary domains: \emph{Action-Labeled Data} for cross-embodiment physical execution priors, and \emph{Visual-Reasoning Data} for semantic and spatial grounding.

\subsection{Action-Labeled Data}
To translate visual understanding into robust physical execution, we build an extensive collection of embodied action-labeled trajectories. Characterized by heterogeneous hardware and long-horizon tasks, the action labeled subset is primarily aggregated from four sources:

\begin{itemize}
    \item \textbf{AgiBotWorld}~\citep{contributors2024agibotworldrepo}. Comprising over one million trajectories across 217 tasks in five deployment scenarios, AgiBotWorld covers contact-rich manipulation, long-horizon planning, and dual-arm collaboration equipped with end-effectors from grippers to dexterous hands. Besides, it provides detailed fine-grained sub-task annotations.
    
    \item \textbf{RoboMind}~\citep{wu2024robomind}.
    As a large-scale, unified teleoperation data dataset, RoboMind introduces a vast array of demonstration trajectories across 479 diverse tasks involving 96 object classes. It enriches the physical priors by providing standardized, high-quality trajectories across diverse scenes and complementing AgiBotWorld in task breadth.
    
    \item \textbf{Open X-Embodiment}~\citep{open_x_embodiment_rt_x_2023}. Open X-Embodiment Dataset \citep{open_x_embodiment_rt_x_2023} encompasses 22 robotic embodiments across distinct laboratory and real-world environments, which is essential for endowing the model with strong cross-embodiment generalization capabilities.
    
    \item \textbf{PRTS Self-Collected Dataset.} Collected on dual-arm RealMan and Agilex platform, our dataset covers a wide range of daily manipulation tasks and complements the dataset with highly complex, multi-scene scenarios that demand precise, contact-rich operations, enhancing the diversity of the training distribution.
\end{itemize}

\subsection{Visual-Reasoning Data}
The visual-reasoning subset supplies grounding, planning, and general vision--language competence that are often under-represented in action-labeled demonstrations.
We organize this reasoning corpus into three hierarchical capability tiers:

\textbf{Spatial Grounding and Affordance Perception}. To safely and accurately interact with the physical world, the model needs to achieve pixel-level localization and understand operable object regions. We utilize \textbf{RefCOCO} \citep{kazemzadeh-etal-2014-referitgame} for fundamental referring expression comprehension, pairing natural language descriptions with precise 2D bounding boxes.
To further enhance fine-grained localization, we incorporate \textbf{Pixmo-Point} \citep{deitke2024molmopixmoopenweights}, which employs a Point-QA paradigm, forcing the model to output precise 2D pixel coordinates rather than abstract text. This precise localization enables fine-grained visual grounding, such as pointing-based counting and visual explanation.

To bridge the domain gap to embodied workspaces, we incorporate \textbf{RoboPoint} \citep{yuan2024robopoint} for spatial-relation grounding alongside standard object detection. 
We further include \textbf{RoboRefIt} \citep{lu2023vl} to resolve referential ambiguity in cluttered environments, encouraging the model to track correct grasping targets under complex relative spatial relations. 
We also incorporate \textbf{RefSpatial} \citep{zhou2025roborefer} to improve 3D scene understanding and enable dynamic reasoning about instruction-specified locations for interaction. 
Moreover, dexterous manipulation inherently requires understanding functional object parts. 
To this end, we incorporate \textbf{RoboAfford} \citep{inproceedings} to endow the model with physical affordance perception, enabling it to distinguish operable regions (e.g., the handle of a knife versus the blade).

\textbf{Task Planning and High-level Reasoning}.
Long-horizon execution requires temporal foresight and physical common sense. To this end, we inject \textbf{Cosmos-Reason1} \citep{nvidia2025cosmosreason1physicalcommonsense} to provide dense reasoning chains (Chain-of-Thought) regarding physical causality and scene dynamics. For long-horizon execution, we leverage \textbf{EgoPlanIT} \citep{chen2024egoplanbenchbenchmarkingmultimodallarge} to train the model in decomposing abstract human instructions into sequential, actionable sub-goals from an egocentric perspective. Additionally, \textbf{RoboVQA} \citep{10610216} is included to enhance step-by-step temporal reasoning and context awareness over continuous embodied video streams.

\textbf{General Semantic Understanding}.
To retain general vision--language competence during large-scale embodied pre-training, we include a curated, high-quality subset of \textbf{LLaVA-Instruct} \citep{liu2024llavanext}. This subset acts as a semantic regularizer, preserving robust zero-shot language-following abilities, visual world knowledge, and conversational fluidity, ensuring the VLA backbone model remains a highly capable conversational agent.

%% file: sec/5_related_work.tex
\section{Related Work}



\textbf{Visual-Language-Action Models. } 
Early VLAs such as RT-1~\citep{brohan2022rt1}, RT-2~\citep{zitkovich2023rt2}, and OpenVLA~\citep{kim2024openvla} formulate policy learning as next-token prediction over discretized action tokens. 
While effective for action prediction, this objective provides only weak, indirect supervision for learning rich visual-semantic representations beyond behavior cloning.
Therefore, current VLA training pipeline typically consists of large-scale pre-training over the mixture of action-labeled data and visual-reasoning data~\citep{black2025pi05,pi-star-0.6, bjorck2025gr00tn1, zhai2025wallx}, followed by task- or embodiment-specific post-training~\citep{kim2025openvla-oft, zhang2025ate, yang2025taco}.
PRTS specifically tackles the representation bottleneck during the pre-training phase.
To enrich the pre-training representations and further improve action generation, recent systems incorporate auxiliary VLM-side training objectives: $\pi_{0.5}$~\citep{black2025pi05} and GR00T~N1~\citep{bjorck2025gr00tn1} add VQA, sub-task prediction, and discrete-action classification to the backbone, whereas RoboBrain~\citep{robobrain}, RoboVLM~\citep{robovlm}, and CoT-VLA~\citep{Zhao2025CoTVLAVC} inject chain-of-thought reasoning and task decomposition.
However, while these auxiliary objectives improve static visual-semantic grounding, they fail to encode the \emph{temporal structure of goal-reaching}---how a task unfolds along a trajectory toward its language goal.
Most closely related to our approach are value-augmented VLAs such as $\pi^*_{0.6}$~\citep{pi-star-0.6}, VLAC~\citep{VLAC}, which attach separate value networks to affect action label supervision. 
A fundamental limitation of these approaches is that their value modules are architecturally decoupled from the VLM backbone.
Consequently, gradients from value learning cannot directly shape the shared vision-language representation. 
In contrast, PRTS integrates contrastive reinforcement learning into VLM pre-training itself, allowing goal-conditioned value supervision to co-train the same representation used for both semantic reasoning and action generation. 
Moreover, existing value-augmented methods also rely on Monte Carlo (MC) returns or temporal-difference (TD) bootstrapping. 
MC estimation suffers from high variance in long-horizon, sparse-reward tasks, while TD learning accumulates function-approximation error~\citep{kumar2019stabilizing, kumar2020cql}. 
PRTS instead formulates value learning as a robust contrastive classification problem with temporally weighted positives, enabling stable and efficient learning of temporal-aware, goal-conditioned representations.


\textbf{Contrastive Reinforcement Learning.}
Goal-conditioned reinforcement learning (GCRL) aims to learn policies that reach arbitrary target states. Classical approaches such as Hindsight Experience Replay~\citep{andrychowicz2017her} and C-Learning~\citep{eysenbach2021clearning} address the sparse-reward challenge of GCRL through goal relabeling or classifier-based value estimation, yet still inherit the instability of TD bootstrapping under function approximation~\citep{sutton2018reinforcement, kumar2019stabilizing}.
Contrastive RL (CRL)~\citep{eysenbach2022contrastiverl} closes this gap by reformulating value learning as an InfoNCE classification problem~\citep{oord2018cpc}: the goal-conditioned value function under geometric discounting is recoverable purely by cross-entropy classification, without Bellman regression. \citet{zheng2024stabilizing} further develops offline-data techniques that make CRL practical for robotic goal reaching.
The central practical point is that CRL replaces the regression-based TD/MC loss with a cross-entropy classification loss, a swap that matters for scale. \citet{1000-layer} shows that CRL trains stably over one thousand layers, a depth at which TD-regression losses provide no meaningful learning signal, and classification-based value estimation has been shown to scale better than regression across deep RL more broadly~\citep{stop-regression}.
These two properties make CRL a particularly natural fit for VLA pre-training. {(i)} VLMs are natively optimized with a cross-entropy token-prediction objective, so CRL's classification loss composes with the backbone without introducing a regression head or auxiliary optimization regime. {(ii)} A VLA is itself a goal-conditioned RL problem: the language instruction specifies the goal, and the policy must reach states that satisfy it. PRTS builds on this alignment, integrating contrastive goal-reaching objectives directly into VLM pre-training so that semantic reasoning and value-aware planning are learned within a single cross-entropy framework.

%% file: sec/6_exp.tex
\section{Experiments}\label{sec:exp}
PRTS is designed to perform robustly and generalize effectively across a wide range of embodied tasks including long-horizon tasks and settings with distribution shifts.
To this end, we evaluate PRTS through a comprehensive set of experiments spanning simulation and real-world deployment, robustness to distribution shifts, instruction-following generalization, value analysis, and pre-training efficiency.
Through extensive experiments, we aim to answer the following key questions:

\textsc{\bf Q1:}
\textbf{Overall Performance.}
How does PRTS compare with state-of-the-art VLAs in both simulation and real-world settings? (Sec.~\ref{sec:benchmark})

\textsc{\bf Q2:}
\textbf{Effect of Contrastive RL.} 
How does incorporating contrastive RL during pre-training contribute to the final performance of PRTS when the post-training recipe is held fixed? (Sec.~\ref{sec:ablation})

\textsc{\bf Q3:}
\textbf{Generalization Capability.} 
How well does PRTS generalize under distribution shifts, including variations in spatial position, lighting conditions, object appearance, and even novel instructions, i.e., zero-shot instruction-following? (Sec.~\ref{sec:generalization})

\textsc{\bf Q4:}
\textbf{Robustness and Recovery under Human Interventions.} 
Can PRTS maintain robust performance and recover from failures under frequent human interventions during execution? (Sec.~\ref{sec:robustness})

\textsc{\bf Q5:}
\textbf{Value Representation.} 
Do the learned representations faithfully encode goal-conditioned action values? (Sec.~\ref{sec:value})

\textsc{\bf Q6:}
\textbf{Pre-Training Efficiency.}
How efficient are our optimized pre-training infrastructure and framework for scaling contrastive RL? (Sec.~\ref{sec:efficiency})


\subsection{Experimental Setup}\label{sec:setup}

\textbf{Pre-training recipe.}
We initialize PRTS from Qwen3-VL-4B-Instruct~\citep{Qwen3-vl} and pre-train it on the corpus described in Sec.~\ref{sec:dataset_overview} using the objective in Eq.~\eqref{eq:loss-pre}. We set the contrastive-RL coefficient to $\lambda_{\text{crl}}=1.0$ and the temporal discount in the contrastive weights to $\gamma=0.995$.
To improve training efficiency, we pack multiple short sequences into a single sequence of length $4096$, reducing wasted computation on padding tokens.
Pre-training runs for one epoch on $64\;\times$ H100 GPUs with a global batch size of $256$ packed sequences, totaling approximately $220$K gradient steps and finishing in about one week.
Finally, thanks to the sharded contrastive similarity computation strategy, we obtain roughly $2$K in-batch negatives for the CRL objective, which is critical for scaling contrastive RL during pre-training.

\textbf{Post-training recipe.}
For each downstream benchmark, we attach a randomly initialized flow-matching DiT action expert to the pre-trained backbone and optimize the post-training objective in Eq.~\eqref{eq:fm-loss}.
The action expert has $675$M parameters and uses $5$ flow-matching denoising steps at inference.
Table~\ref{tab:posttrain_recipes} summarizes the post-training schedules.
Across both simulation and real-robot evaluations, PRTS uses a post-training compute budget that is no larger than, and in most cases substantially smaller than, those used by strong representative VLAs such as $\pi_{0.5}$.
For example, comparisons on LIBERO simulation benchmark use only $1/8$ of the post-training compute used by $\pi_{0.5}$.
This design makes performance differences easier to attribute to how well the pre-trained VLA acts as a general visuomotor prior rather than to a larger downstream fine-tuning budget.

\begingroup
\captionsetup{justification=centering, singlelinecheck=true}
\begin{table}[htbp]
\centering
\renewcommand{\arraystretch}{1.15}
\begin{threeparttable}
\resizebox{\textwidth}{!}{
\begin{tabular}{l c c c}
\toprule
\textbf{Benchmark} & \textbf{Global batch size} & \textbf{Gradient steps} & \textbf{Action chunk $H$}\\
\midrule
LIBERO  & $32$    & $30$K  & $20$ \\
SimplerEnv (WidowX)                 & $1024$ & $20$K & $16$ \\
RealMan Dual-Arm Robot Platform (all-task fine-tuning) & $64$  & $100$K & $20$ \\
Flexiv Robot Platform (single-task fine-tuning) & $32$ & $40$K &  $20$ \\
\bottomrule
\end{tabular}
}
\caption{\textbf{Post-training recipes per benchmark.}}
\label{tab:posttrain_recipes}
\end{threeparttable}
\end{table}
\endgroup

\textbf{Real-world robot platforms.}
Our real-world evaluation uses the two platforms shown in Fig.~\ref{fig:hardware_overview}: a Flexiv platform for three single-arm manipulation tasks and a dual-arm {RealMan} platform for eleven bimanual manipulation tasks.
The RealMan platform uses a stationary $14$-DoF bimanual configuration, with each arm equipped with a parallel-jaw gripper.
It observes the workspace through three RGB cameras: one egocentric head camera and one wrist camera per arm.
The Flexiv platform, shown in the right panel of Fig.~\ref{fig:hardware_overview}, provides a complementary single-arm setting.

\begin{figure}[htbp]
    \centering
    \includegraphics[width=\linewidth]{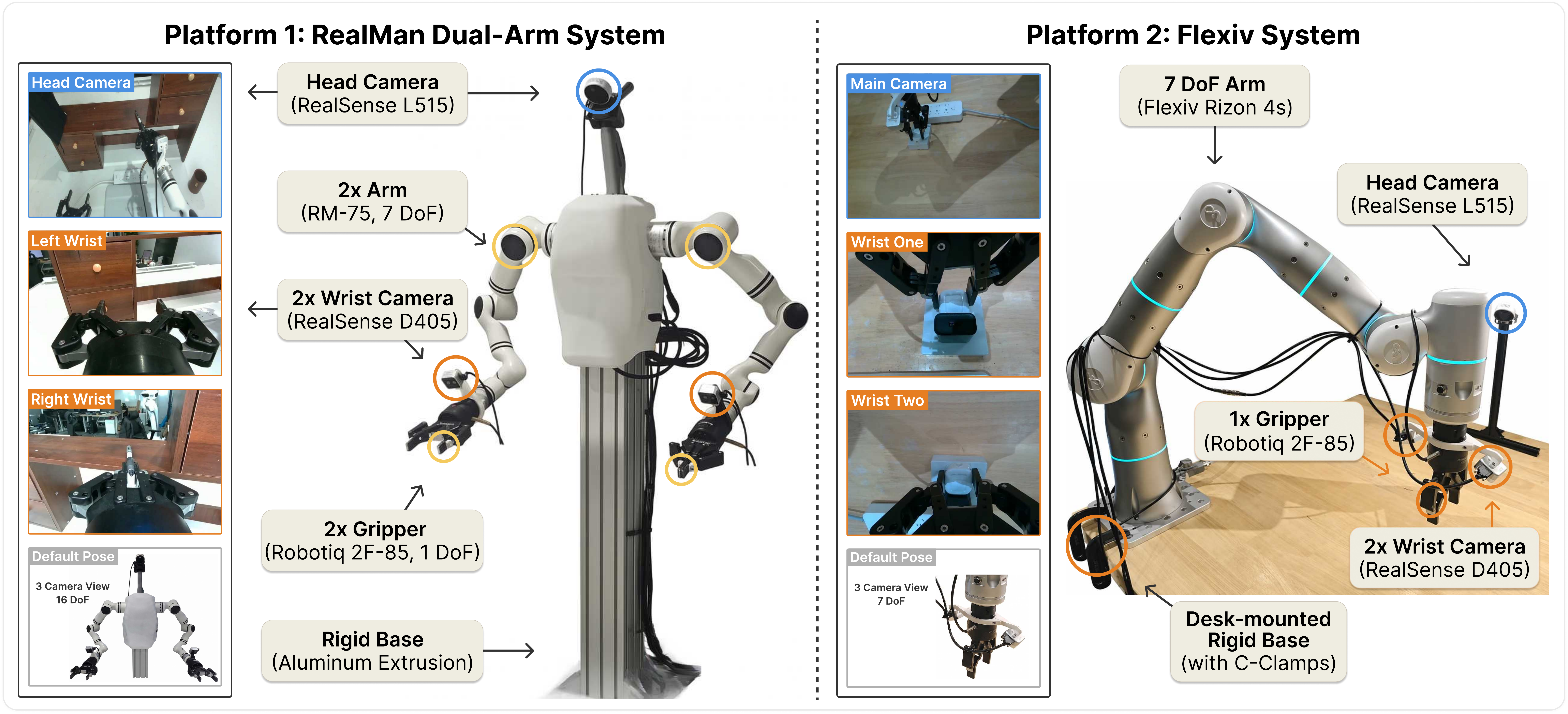}
    \caption{\textbf{Overview of real-world robot platforms.} Left: dual-arm RealMan platform with two parallel-jaw grippers, one egocentric head camera, and two wrist cameras (one per arm). Right: Flexiv single-arm platform equipped with a parallel-jaw gripper, a third-view head camera, and two wrist cameras.}
    \label{fig:hardware_overview}
\end{figure}

\textbf{Benchmarks.}
We evaluate PRTS on five benchmarks:
\begin{itemize}
    \item \textbf{LIBERO}~\citep{liu2023libero}, the four standard suites (\emph{Spatial}, \emph{Object}, \emph{Goal}, \emph{Long}) for measuring in-distribution visuomotor accuracy.
    \item \textbf{LIBERO-Plus}~\citep{fei2025liberoplus}, a robustness-oriented extension with $10{,}030$ test tasks spanning seven perturbation dimensions: object layout, camera viewpoint, robot initial state, language instruction, lighting, background texture, and sensor noise.
    \item \textbf{LIBERO-Pro}~\citep{zhou2025liberopro}, a perturbation-based extension of LIBERO that tests whether a policy follows the task specification rather than memorizing the original scene, covering five generalization dimensions: \emph{Object}, \emph{Position}, \emph{Semantic}, \emph{Task}, and \emph{Environment}.
    \item \textbf{SimplerEnv (WidowX)}~\citep{simpler}, the BridgeData V2/WidowX track of SIMPLER, which evaluates policies trained on real-world data in simulation through four rigid-object pick-place-stack tasks.
    \item \textbf{Real-world evaluation}, three fine-grained manipulation tasks on the Flexiv platform and eleven manipulation tasks on the dual-arm RealMan platform, together with deliberate perturbations that vary lighting, initial object position, object identity, and task instruction at deployment time. We visualize all these task in Figs.~\ref{fig:realman_task_details} and~\ref{fig:flexiv_task_details} and provide detailed descriptions in Appendix~\ref{app:task_details}.
\end{itemize}

\textbf{Baselines.}
The baseline set differs across benchmarks because we use strong representative VLAs reported under each protocol, whereas our real-world experiments directly compare against $\pi_0$~\citep{black2024pi0} and $\pi_{0.5}$~\citep{black2025pi05} under the same evaluation protocol.
Also, we particularly annotate the post-training budget of those competitive VLA baselines in gray next to the model name when available, to make the compute budget explicit.
This annotation aims to clearly highlight how rich a general visuomotor prior PRTS provides compared to existing VLAs, particularly under the same or even smaller post-training budget.

\textbf{Evaluation protocol.}
All simulated benchmarks report task success rate (SR) under the standard protocol of the corresponding benchmark.
For real-world evaluation, each task cell is evaluated over $20$ physical trials.
On the Flexiv platform, a trial is marked as failed if the robot does not succeed within five grasp attempts or within two minutes.
On the RealMan platform, grasping tasks similarly allow up to five feasible re-grasp attempts and are marked as failed if they exceed a three-minute time limit.

\subsection{Overall Performance}
\label{sec:benchmark}
We first evaluate PRTS on four simulation benchmarks that probe complementary aspects of policy competence: in-distribution visuomotor execution on LIBERO, robustness to distribution shifts on LIBERO-Plus and LIBERO-Pro, instruction-grounded generalization on LIBERO-Pro, and cross-embodiment adaptability to the WidowX robot in SimplerEnv.
For LIBERO, PRTS is post-trained with global batch size $32$ for $30$K gradient steps.
The resulting checkpoint is then evaluated \emph{zero-shot} on LIBERO-Plus and LIBERO-Pro, without any additional training on those benchmark-specific perturbation datasets.
For SimplerEnv, PRTS is post-trained with global batch size $1024$ for $20$K gradient steps, same as the compute budget of GR00T-N1.5~\citep{bjorck2025gr00tn1}, the strongest prior results on that benchmark.
The baseline set differs between benchmarks and is listed in each table.

\subsubsection{LIBERO}
Table~\ref{tab:libero_benchmark} compares PRTS against strong VLA baselines on the four standard LIBERO suites.
With only global batch size $32$ and $30$K post-training steps, PRTS reaches $98.4\%$ average SR and matches the state-of-the-art performance of ABot-M0 while using much less post-training compute than several strong baselines. In particular, $\pi_{0.5}$ uses an $8\times$ larger batch (bs=$256$, $30$K steps) yet reaches $96.9\%$, while the strongest ABot-M0 setting uses both a larger batch and more steps (bs=$64$, $40$K steps) to reach $98.6\%$. To separate representation quality from post-training compute, we also train ABot-M0 under the same bs=$32$ and $30$K gradient steps budget as PRTS. Under this matched budget, ABot-M0 reaches $97.9\%$ average SR, $0.5$ points below PRTS, and suffers a dramatic drop from $96.6\%$ to $94.8\%$ on the \textit{Long} suite.
PRTS, in contrast, reaches $96.6\%$ on \textit{Long}, matching the best reported result with significantly smaller post-training cost.
It suggests that the goal-reachability-aware representation learned through contrastive RL pre-training is especially critical for long-horizon manipulation tasks.
Overall, the controlled-budget comparison indicates that PRTS benefits from a stronger pre-trained visuomotor prior, rather than relying on heavier downstream optimization to recover missing temporal reasoning ability.

\begin{table}[htbp]
\centering
\renewcommand{\arraystretch}{1.0} 
\begin{threeparttable}
\begin{tabular}{l c c c c c}
\toprule
\textbf{Method} &  \textbf{Spatial} & \textbf{Object} & \textbf{Goal} & \textbf{Long} & \textbf{Average} \\
\midrule
Diffusion Policy~\citep{chi2023diffusionpolicy}       & 78.5 & 87.5 & 73.5 & 64.8 & 76.1 \\
OpenVLA~\citep{kim2024openvla}                & 84.7 & 88.4 & 79.2 & 53.7 & 76.5 \\
SpatialVLA~\citep{qu2025spatialvla}             & 88.2 & 89.9 & 78.6 & 55.5 & 78.1 \\
CoT-VLA~\citep{Zhao2025CoTVLAVC}                & 87.5 & 91.6 & 87.6 & 69.0 & 83.9 \\
GR00T-N1~\citep{bjorck2025gr00tn1}               & 94.4 & 97.6 & 93.0 & 90.6 & 93.9 \\
F1-VLA~\citep{lv2025f1}                     & 98.2 & 97.8 & 95.4 & 91.3 & 95.7 \\
InternVLA-M1~\citep{chen2025internvlam1}          & 98.0 & 99.0 & 93.8 & 92.6 & 95.9 \\
Discrete Diffusion VLA~\citep{liang2025discretedvla} & 97.2 & 98.6 & 97.4 & 92.0 & 96.3 \\
GR00T-N1.6~\citep{bjorck2025gr00tn1}             & 97.7 & 98.5 & 97.5 & 94.4 & 97.0 \\
OpenVLA-OFT~\citep{kim2025openvla-oft}            & 97.6 & 98.4 & 97.9 & 94.5 & 97.1 \\
X-VLA~\citep{zheng2025xvla}                  & 98.2 & 98.6 & 97.8 & \textbf{97.6} & 98.1 \\
$\pi_0$-Fast~\citep{pertsch2025fast} {\scriptsize \color{gray} (bs=32, 30K steps)} & 96.4 & 96.8 & 88.6 & 60.2 & 85.5 \\
$\pi_0$~\citep{black2024pi0} {\scriptsize \color{gray} (bs=32, 30K steps)}   & 96.8 & 98.8 & 95.8 & 85.2 & 94.2 \\
Qwen3-VL-PI~\citep{starvla2025} {\scriptsize \color{gray} (bs=32, 30K steps)} & 95.2 & 99.0 & 96.2 & 88.4 & 94.7 \\  
$\pi_{0.5}$~\citep{black2025pi05} {\scriptsize \color{gray} (bs=256, 30K steps)}            & \textbf{98.8} & 98.2 & 98.0 & 92.4 & 96.9 \\    
ABot-M0~\citep{yang2026abotm0} {\scriptsize \color{gray} (bs=32, 30K steps)}               & \underline{98.6} & \underline{99.6} & \underline{98.6} & 94.8 & 97.9 \\
ABot-M0~\citep{yang2026abotm0} {\scriptsize \color{gray} (bs=64, 40K steps)}               & \textbf{98.8} & \textbf{99.8} & \textbf{99.0} & \underline{96.6} & \textbf{98.6} \\
\midrule 
\textbf{PRTS (Ours)} {\scriptsize \color{gray} \textbf{(bs=32, 30K steps)}} & \textbf{98.8} & \textbf{99.8} & 98.4 & \underline{96.6} & \underline{98.4} \\
\bottomrule
\end{tabular}
\caption{\textbf{Evaluation results on the LIBERO benchmark.} Post-training budget is indicated in gray next to the model names. Best results are in {bold}, and second-best results are {underlined}. PRTS matches the strongest prior results with global batch size $32$ and $30$K post-training steps; ABot-M0 under the same budget is included for a controlled comparison.}
\label{tab:libero_benchmark}
\end{threeparttable}
\end{table}

\subsubsection{LIBERO-Plus}
Table~\ref{tab:libero_plus_benchmark} reports zero-shot evaluation on LIBERO-Plus using the same checkpoint post-trained only on standard LIBERO. LIBERO-Plus introduces seven controlled perturbation axes, so performance reflects whether a policy can preserve task execution under shifted camera views, robot states, language, lighting, backgrounds, sensor noise, and object layouts. PRTS achieves the best average SR ($81.4\%$), outperforming $\pi_{0.5}$ ($80.7\%$) despite using an $8\times$ smaller LIBERO post-training batch.
More importantly, the gains are concentrated on quite challenging perturbations: \textit{Robot} ($+14.4$), \textit{Background} ($+15.3$), and \textit{Language} ($+1.1$), while remaining competitive across all other axes.
These results show that the CRL pre-trained backbone provides a robust goal-conditioned prior.
After fine-tuning on standard LIBERO, the policy can still identify and execute the intended task under changes in robot state, scene appearance, and LLM-rewritten instructions.

\begin{table}[htbp]
\centering
\renewcommand{\arraystretch}{1.1} 
\begin{threeparttable}

\resizebox{\textwidth}{!}{ 
\begin{tabular}{l c c c c c c c c}
\toprule
\textbf{Method} & \textbf{Camera} & \textbf{Robot} & \textbf{Language} & \textbf{Light} & \textbf{Background} & \textbf{Noise} & \textbf{Layout} & \textbf{Average} \\
\midrule
OpenVLA~\citep{kim2024openvla}            & 0.8  & 3.5  & 23.0 & 8.1  & 34.8 & 15.2 & 28.5 & 15.6 \\
WorldVLA~\citep{cen2025worldvla}           & 0.1  & 27.9 & 41.6 & 43.7 & 17.1 & 10.9 & 38.0 & 25.0 \\
NORA~\citep{hung2025nora}               & 2.2  & 37.0 & 65.1 & 45.7 & 58.6 & 12.8 & 62.1 & 39.0 \\
UniVLA~\citep{bu2025univla}             & 1.8  & 46.2 & 69.6 & 69.0 & 81.0 & 21.2 & 31.9 & 42.9 \\
RIPT-VLA~\citep{tan2025interactive}           & 55.2 & 31.2 & 77.6 & 88.4 & 91.6 & 73.5 & 74.2 & 68.4 \\
OpenVLA-OFT~\citep{kim2025openvla-oft}        & 56.4 & 31.9 & 79.5 & 88.7 & 93.3 & 75.8 & 74.2 & 69.6 \\
$\pi_0$~\citep{black2024pi0} {\scriptsize \color{gray} (bs=32, 30K steps)} & 13.8 & 6.0  & 58.8 & 85.0 & 81.4 & 79.0 & 68.9 & 53.6 \\
$\pi_0$-Fast~\citep{pertsch2025fast} {\scriptsize \color{gray} (bs=32, 30K steps)} & \underline{65.1} & 21.6 & 61.0 & 73.2 & 73.2 & 74.4 & 68.8 & 61.6 \\
Qwen3-VL-PI~\citep{starvla2025} {\scriptsize \color{gray} (bs=32, 30K steps)} & 43.6 & 46.6 & 81.8 & 88.4 & 89.4 & 
70.1 & 74.7 & 68.9 \\
ABot-M0~\citep{yang2026abotm0} {\scriptsize \color{gray} (bs=32, 30K steps)}               & \textbf{65.7} & 53.1 & 81.2 & \textbf{96.8} & \underline{95.0} & \textbf{87.5} & 81.3 & 78.7  \\
$\pi_{0.5}$~\citep{black2025pi05} {\scriptsize \color{gray} (bs=256, 30K steps)} & 64.0 & \underline{58.0}  & \underline{88.5} & \underline{96.6} & 81.4 & \textbf{87.5} & \textbf{85.9} & \underline{80.7} \\
\midrule
\textbf{PRTS (Ours)} {\scriptsize \color{gray} \textbf{(bs=32, 30K steps)}} & 59.5 & \textbf{72.4} & \textbf{89.6} & 96.5 & \textbf{96.7} & \underline{81.3} & \underline{83.3} & \textbf{81.4} \\
\bottomrule
\end{tabular}
}
\caption{\textbf{Evaluation results on the LIBERO-Plus benchmark.} Zero-shot robustness evaluation of the LIBERO-finetuned checkpoint across seven perturbation axes~\citep{fei2025liberoplus}. Post-training budget is indicated in gray when available. Best results are in {bold}, and second-best results are {underlined}.}
\label{tab:libero_plus_benchmark}
\end{threeparttable}
\end{table}

\subsubsection{LIBERO-Pro}
LIBERO-Pro (Table~\ref{tab:libero_pro_benchmark}) further tests whether the LIBERO-finetuned checkpoint follows the task specification rather than replaying memorized trajectories. We again perform zero-shot evaluation.
This benchmark was designed to expose a key flaw of current VLAs: a high success rate on standard LIBERO can stem from mechanical memorization of training scenarios rather than transferable task-solving strategies~\citep{zhou2025liberopro}. To this end, LIBERO-Pro introduces perturbation axes that directly test whether a policy can solve the specified task rather than replay a familiar trajectory.\footnote{Here we do not include the \textit{Environment} perturbation because the official codebase has not yet released the corresponding assets.}
The \textit{Semantic} axis paraphrases the instruction while preserving the task intent, similar to the \textit{Language} perturbations in LIBERO-Plus, and the \textit{Object} axis changes object appearance and size.
Notably, the most distinctive axes are \textit{Position} and \textit{Task}. \textit{Position} swaps the relative locations of the pick and place objects, going beyond a simple initial-position shift: the policy must infer the operational relation between the two objects and recompute the corresponding action sequence. \textit{Task} keeps the scene fixed but gives a novel instruction that asks the policy to manipulate a different target object, directly testing zero-shot novel instruction following.
PRTS reaches $58.8\%$ average SR, improving over $\pi_{0.5}$ by $+5.5$ points.
The most significant improvement appears on \textit{Task}: PRTS reaches $31.5\%$, while $\pi_{0.5}$ drops to $0.8\%$ and the second-best VLA reaches only $22.3\%$, establishing a substantial absolute margin of $9.2$ points.
PRTS also leads on \textit{Position} ($24.3\%$ vs. $20.8\%$ for $\pi_{0.5}$). Especially, we provide a more detailed numerical results on these two axes in Table~\ref{tab:additional_libero_pro} in Appendix~\ref{app:libero_pro_detail}.
Together, these two axes highlight the core advantage of our pre-training design: rather than merely imitating demonstration trajectories, PRTS learns goal-conditioned representations that encode which state-action pairs are truly reachable under a given language instruction.
This enables goal-reachability-aware decision making even when object layouts change or entirely new task instructions are introduced, precisely the capability targeted by contrastive RL pre-training.

\begin{table}[htbp]
\centering
\renewcommand{\arraystretch}{1.1} 
\begin{threeparttable}
\resizebox{\textwidth}{!}{ 
\begin{tabular}{l c c c c c}
\toprule
\textbf{Method} & \textbf{Semantic} & \textbf{Object} & \textbf{Position} & \textbf{Task} & \textbf{Average}\\
\midrule
OpenVLA~\citep{kim2024openvla}           & \textbf{97.2}  & \underline{93.0}  & 0.0 & 0.0 & 47.6 \\
Qwen3-VL-PI~\citep{starvla2025} {\scriptsize \color{gray} (bs=32, 30K steps)} & 95.5 & 70.4 & 4.3 & 3.9 & 43.5 \\
ABot-M0~\citep{yang2026abotm0} {\scriptsize \color{gray} (bs=32, 30K steps)} & \underline{97.1} & 82.5 & 7.1 & \underline{22.3} & 52.2\\
$\pi_0$~\citep{black2024pi0} {\scriptsize \color{gray} (bs=32, 30K steps)} & 90.5 & 90.5  & 0.0 & 0.0 & 45.3 \\
$\pi_{0.5}$~\citep{black2025pi05} {\scriptsize \color{gray} (bs=256, 30K steps)} & 95.8 & \textbf{96.0} & \underline{20.8} & 0.8 & \underline{53.3} \\
\midrule
\textbf{PRTS (Ours)} {\scriptsize \color{gray} \textbf{(bs=32, 30K steps)}} & 97.0 & 82.3 & \textbf{24.3} & \textbf{31.5} & \textbf{58.8} \\
\bottomrule
\end{tabular}
}
\caption{\textbf{Evaluation results on the LIBERO-Pro benchmark.} Zero-shot generalization evaluation of the LIBERO-finetuned checkpoint across semantic, object, position, and task variations~\citep{zhou2025liberopro}. Post-training budget is indicated in gray when available. Best results are in {bold}, and second-best results are {underlined}.}
\label{tab:libero_pro_benchmark}
\end{threeparttable}
\end{table}

\subsubsection{SimplerEnv with Visual Matching}
Table~\ref{tab:simpler_benchmark} evaluates PRTS on the WidowX robot in SimplerEnv, covering three pick-and-place tasks and one stacking task.
PRTS is post-trained on the real-world BridgeData V2 dataset~\citep{walke2023bridgedata} with the same global batch size and number of gradient steps as GR00T-N1.5, the strongest prior result on this benchmark.
PRTS reaches $77.1\%$ average SR, improving over the strongest prior result GR00T-N1.5 by $+15.2$ points and over starVLA based on the same VLM backbone Qwen3-VL-4B-Instruct by $+16.2$ points. Notably, PRTS also reaches $100\%$ on \textit{Put Eggplant in Yellow Basket} while leading on \textit{Put Spoon on Towel} and \textit{Stack Green Block on Yellow Block}.
These results show that PRTS can also achieve strong performance on a robot embodiment different from LIBERO, demonstrating its adaptability and generality across embodiments.


\begin{table}[htbp]
\centering
\renewcommand{\arraystretch}{1.1} 
\begin{threeparttable}
\resizebox{\textwidth}{!}{ 
\begin{tabular}{l c c c c c}
\toprule
\textbf{Method} &  \textbf{\makecell{Put Spoon\\On Towel}} & \textbf{\makecell{Put Carrot\\On Plate}} & \textbf{\makecell{Stack Green Block\\On Yellow Block}} & \textbf{\makecell{Put Eggplant\\In Yellow Basket}} & \textbf{Average} \\
\midrule
OpenVLA~\citep{kim2024openvla} & 0.0 & 0.0 & 0.0 & 4.1 & 1.0 \\
RT-1-X~\citep{openxembodiment} & 0.0 & 4.2 & 0.0 & 0.0 & 1.1 \\
OpenVLA-OFT~\citep{kim2025openvla-oft}            & 34.2 & 30.0 & 30.0 & 72.5 & 41.8 \\
SpatialVLA~\citep{qu2025spatialvla}             & 16.7 & 25.0 & 29.2 & \textbf{100.0} & 42.7 \\
CogACT~\citep{li2024cogact}                 & 71.7 & 50.8 & 15.0 & 67.5 & 51.3 \\
F1-VLA~\citep{lv2025f1}                 & 50.0 & \textbf{70.8} & 50.0 & 66.7 & 59.4 \\
Qwen3-VL-PI~\citep{starvla2025}            & \underline{78.1} & 46.9 & 30.2 & \underline{88.5} & 60.9 \\
GR00T-N1.5~\citep{bjorck2025gr00tn1}             & 75.3 & 54.3 & \underline{57.0} & 61.3 & \underline{61.9} \\
$\pi_0$~\citep{black2024pi0}                & 29.1 & 0.0 & 16.6 & 62.5 & 27.1 \\
$\pi_{0.5}$~\citep{black2025pi05} & 29.2 & 41.7 & 0.0 & 41.7 & 28.2\\
$\pi_0$-Fast~\citep{pertsch2025fast}           & 29.1 & 21.9 & 10.8 & 66.6 & 48.3 \\
\midrule 
\textbf{PRTS (Ours)}   & \textbf{83.3} & \underline{66.6} & \textbf{58.3} & \textbf{100.0} & \textbf{77.1} \\
\bottomrule
\end{tabular}
}
\caption{\textbf{Evaluation results on the SimplerEnv WidowX benchmark (Visual Matching).} All numbers are success rates (\%). Best results are in {bold}, and second-best results are {underlined}.}
\label{tab:simpler_benchmark}
\end{threeparttable}
\end{table}

\subsubsection{Real-World Evaluation}
\textbf{RealMan dual-arm system.}
Fig.~\ref{fig:realman_eval} reports success rates (SR) on the $11$-task RealMan Dual-Arm suite. 
To rigorously compare PRTS against other models, we deliberately deviate from the standard per-task fine-tuning paradigm. Instead, we post-train and evaluate a single shared checkpoint across all eleven tasks. This mixed-task post-training setting introduces inevitable cross-task interference, demanding a robust language-conditioned representation to maintain consistent goal-directed behavior across tasks. Furthermore, it establishes the foundation for the zero-shot instruction evaluation in Sec.~\ref{sec:generalization}.

Since all models share identical post-training data and recipes, the performance gap directly reflects the quality of the pre-trained representations. Overall, PRTS achieves a $95.9\%$ average SR, outperforming $\pi_{0.5}$ ($85.5\%$) and $\pi_0$ ($67.3\%$). The advantage of PRTS is especially prononced on tasks requiring precise contact (e.g., \textit{Stack Cups}, \textit{Flip Tennis Tube}) and long-horizon manipulation. Notably, on the two-minute \textit{Office Long Term} task, $\pi_{0.5}$ suffers from multi-task interference and collapses to $40\%$ SR, whereas PRTS maintains a robust $95\%$. This contrast confirms that CRL pre-training injects a temporally consistent goal-reachability prior, preserving stable execution throughout long-horizon rollouts.

\begin{figure}[htbp]
  \centering
  \begin{subfigure}[t]{0.72\textwidth}
    \centering
    \includegraphics[width=\linewidth]{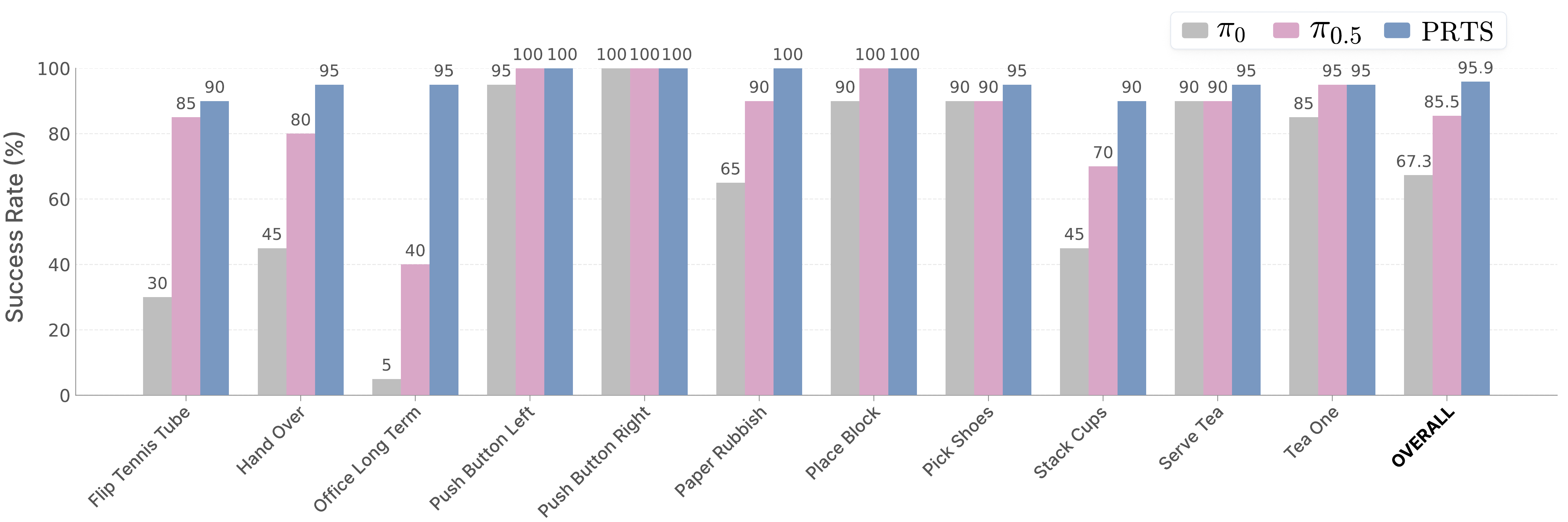}
    \caption{RealMan task suite.}
    \label{fig:realman_eval}
  \end{subfigure}\hfill
  \begin{subfigure}[t]{0.23\textwidth}
    \centering
    \includegraphics[width=\linewidth]{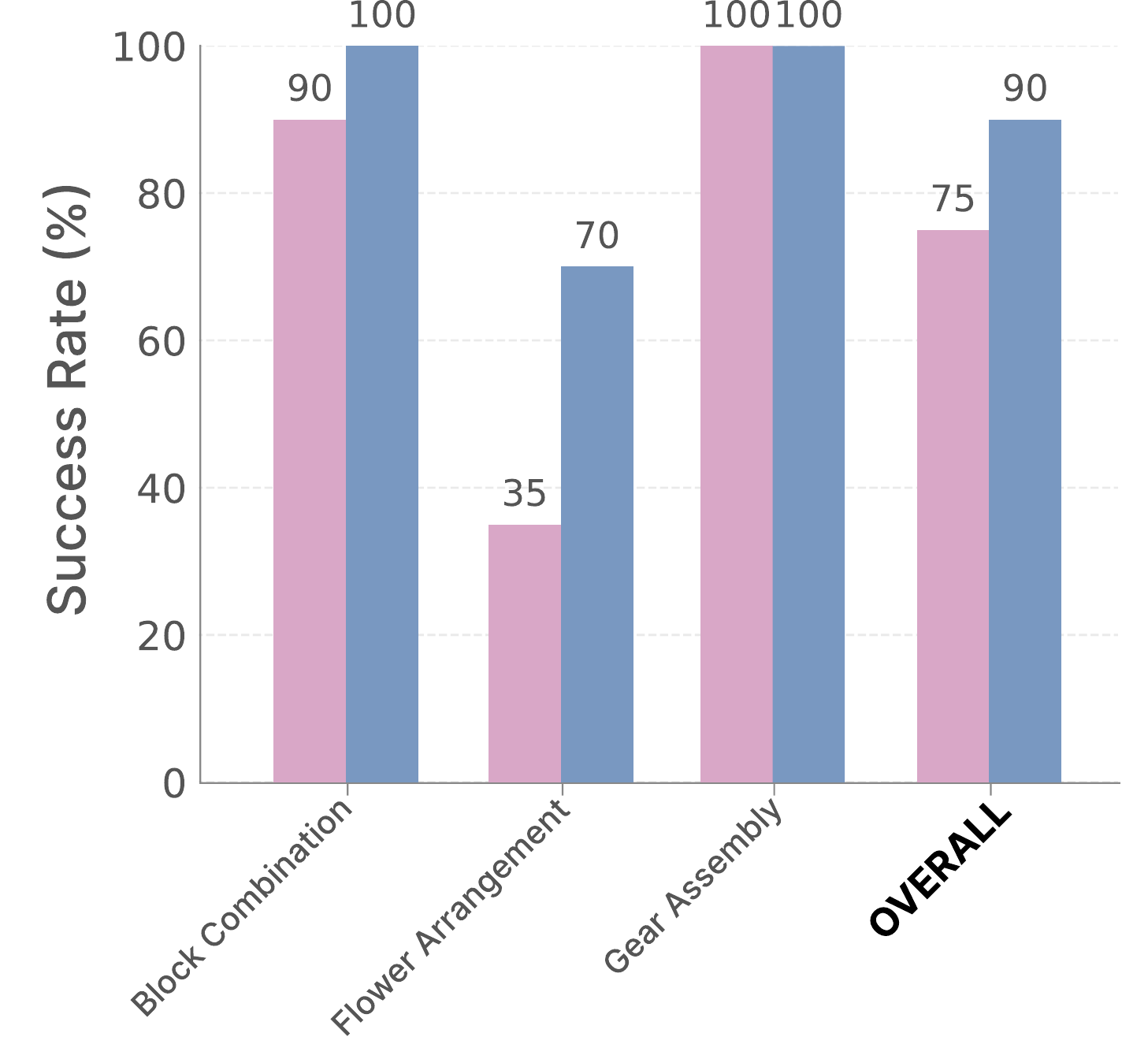}
    \caption{Flexiv task suite.}
    \label{fig:flexiv_eval}
  \end{subfigure}
  \caption{\textbf{Real-world evaluation.} 
  In both suites, per-task success rates averaged over $20$ real-world trials are reported. \textbf{(a) RealMan dual-arm task suite.} PRTS achieves $95.9\%$ average SR and reaches at least $90\%$ on all $11$ tasks, outperforming $\pi_{0.5}$ and $\pi_0$ under the same evaluation protocol. \textbf{(b) Flexiv single-arm task suite.} PPRTS outperforms the baselines in all three tasks spanning spatial-aligned and fine-grained and contact-rich manipulation.}
\end{figure}

\textbf{Flexiv single-arm system.}
To further assess the model's capability for fine-grained, contact-rich manipulation, we evaluate PRTS and baselines on the Flexiv platform under a per-task fine-tuning scheme.
This benchmark emphasizes precise execution through three challenging tasks:
\textit{Gear Assembly}---sequential high-precision insertion of three gears with varying sizes,
\textit{Block Combination}---fine-grained spatial alignment,
and \textit{Flower Arrangement}---a compounding long-horizon task requiring sequential placement of two distinct flowers.
As shown in Fig.~\ref{fig:flexiv_eval}, PRTS achieves a $90.0\%$ average success rate, outperforming $\pi_{0.5}$ by $+15.0\%$.
While both models reach saturated performance on \textit{Gear Assembly} ($100.0\%$), PRTS demonstrates stronger spatial grounding on \textit{Block Combination} ($100.0\%$ vs.\ $90.0\%$), and significantly alleviates the compounding execution errors that cause $\pi_{0.5}$ to struggle on the temporally extended \textit{Flower Arrangement} sequence ($70.0\%$ vs.\ $35.0\%$).
These results suggest that CRL pre-training equips the model with a stronger goal-reachability-aware representation, enabling more reliable action generation under both fine-grained spatial constraints and long-horizon sequential dependencies.

\subsection{Effect of Contrastive RL}\label{sec:ablation}
After the overall comparison, we now examine how much of PRTS's performance comes from the contrastive RL objective used during pre-training. To isolate this effect, we train a matched variant of PRTS using the same architecture, datasets, and training recipe as in Sec.~\ref{sec:setup}, except that the CRL coefficient in Eq.~\eqref{eq:loss-pre} is set to $\lambda_{\text{crl}}=0.0$ during pre-training. After pre-training, we attach the same flow-matching action expert to both PRTS and its non-CRL variant, and post-train them on standard LIBERO with global batch size $32$ for $30$K gradient steps. We evaluate the resulting models on standard LIBERO and further test them zero-shot on LIBERO-Plus and LIBERO-Pro without benchmark-specific adaptation. The results are shown in Table~\ref{tab:crl_ablation_summary}.

\begin{table}[htbp]
\centering
\renewcommand{\arraystretch}{0.9}
\begin{threeparttable}
\small
\begin{tabular*}{\textwidth}{@{\extracolsep{\fill}} l l c c c}
\toprule
\textbf{Benchmark} & \textbf{Suite / Perturbation} & \textbf{PRTS w/o CRL} & \textbf{PRTS w/ CRL} & \textbf{$\Delta$}\\
\midrule
\multirow{5}{*}{\textbf{LIBERO}}
& Spatial & 97.8 & 98.8 & \up{+1.0}\\
& Object & 99.8 & 99.8 & 0.0\\
& Goal & 98.0 & 98.4 & \up{+0.4}\\
& Long & 95.6 & 96.6 & \up{+1.0}\\
& \textbf{Average} & \textbf{97.8} & \textbf{98.4} & \up{\textbf{+0.6}}\\
\midrule
\multirow{8}{*}{\textbf{LIBERO-Plus}}
& Camera & 56.7 & 59.5 & \up{+2.8}\\
& Robot & 56.1 & 72.4 & \up{+16.3}\\
& Language & 87.3 & 89.6 & \up{+2.3}\\
& Light & 96.4 & 96.5 & \up{+0.1}\\
& Background & 95.8 & 96.7 & \up{+0.9}\\
& Noise & 72.3 & 81.3 & \up{+9.0}\\
& Layout & 82.9 & 83.3 & \up{+0.4}\\
& \textbf{Average} & \textbf{76.5} & \textbf{81.4} & \up{\textbf{+4.9}}\\
\midrule
\multirow{5}{*}{\textbf{LIBERO-Pro}}
& Semantic & 96.8 & 97.0 & \up{+0.2}\\
& Object & 84.5 & 82.3 & \down{-2.2}\\
& Position & 13.8 & 24.3 & \up{+10.5}\\
& Task & 20.1 & 31.5 & \up{+11.4}\\
& \textbf{Average} & \textbf{53.8} & \textbf{58.8} & \up{\textbf{+5.0}}\\
\bottomrule
\end{tabular*}
\caption{\textbf{Effect of contrastive RL pre-training across the LIBERO benchmark family.}
Both variants are post-trained on the same standard LIBERO dataset and evaluated zero-shot on LIBERO-Plus and LIBERO-Pro.
All numbers are success rates (\%), and $\Delta$ reports PRTS w/ CRL minus PRTS w/o CRL.}
\label{tab:crl_ablation_summary}
\end{threeparttable}
\end{table}

On standard LIBERO, CRL improves the average SR from $97.8\%$ to $98.4\%$ under the same post-training budget. The strongest improvements appear on \textit{Spatial} and \textit{Long} (both $1.0\%$), suggesting that the CRL-shaped representations benefit relational grounding and long-horizon execution on the in-distribution benchmark.

The advantage becomes much clearer under zero-shot evaluation on LIBERO-Plus.
Specifically, CRL improves the average SR by $4.9\%$, demonstrating substantial gains along the \textit{Robot} ($16.3\%$) and \textit{Noise} ($9.0\%$), alongside notable improvements in the \textit{Camera} ($2.8\%$) and \textit{Language} ($2.3\%$).
Across these permutation dimensions, the policy must execute the intended task consistently despite shifts in proprioceptive initialization, perceptual inputs, or language instructions. Under such distribution shifts, successful execution relies not on memorizing trajectories from post-training, but rather on preserving the language-conditioned goal and effectively re-grounding it within the perturbed observation stream.
By explicitly aligning state representations with high-level task semantics during pre-training, CRL equips the policy with this capability. 
Consequently, PRTS w/ CRL exhibits enhanced robustness when the path from observation to action is disturbed while the underlying task semantics remain fixed.
These results supports our hypothesis that contrastive pre-training inherently strengthens the goal-conditioned representations leveraged by the downstream action expert, moving beyond mere action imitation on the nominal training distribution.

The evaluation on the LIBERO-Pro further corroborates these findings by testing the model's capacity for zero-shot generalization under severe semantic and structural shifts. Here, CRL yields a $5.0\%$ increase in average SR, driven primarily by gains on the \textit{Position} ($10.5\%$) and \textit{Task} ($11.4\%$). 
Unlike lower-level visual variations, these dimensions demand a higher degree of cognitive flexibility. They require the policy to actively reinterpret redefined instruction goals, identify new target relations, and synthesize novel action sequences, rather than simply retrieving and replaying original demonstration trajectories.
The suite-level breakdown in Table~\ref{tab:crl_ablation_libero_pro_detail} provides a more detailed view.
CRL effectively mitigates the severe performance degradation observed in the non-CRL variant under complex re-grounding scenarios. For example, when evaluating on novel spatial configurations (\textit{Object-Position}), the SR improves from $4.8\%$ to $36.0\%$. Furthermore, CRL yields substantial gains in adapting to new task logics, improving \textit{Spatial-Task} and \textit{Goal-Task} performance from $34.4\%$ to $62.2\%$ and $20.2\%$ to $33.8\%$, respectively.

\begin{table}[htbp]
\centering
\renewcommand{\arraystretch}{1.2}
\begin{threeparttable}
\resizebox{\textwidth}{!}{ 
\begin{tabular}{l c c c c c c c c}
\toprule
\multirow{2}{*}{\textbf{Method}} & \multicolumn{2}{c}{\textbf{LIBERO-Spatial}} & \multicolumn{2}{c}{\textbf{LIBERO-Object}} & \multicolumn{2}{c}{\textbf{LIBERO-Goal}} & \multicolumn{2}{c}{\textbf{LIBERO-Long}}\\
\cmidrule(lr){2-3} \cmidrule(lr){4-5} \cmidrule(lr){6-7} \cmidrule(lr){8-9}
& {Pos (Avg.)} & {Task (Avg.)} & {Pos (Avg.)} & {Task (Avg.)} & {Pos (Avg.)} & {Task (Avg.)} & {Pos (Avg.)} & {Task (Avg.)} \\
\midrule
PRTS w/o CRL & 26.8 & 34.4 & 4.8 & 10.0 & 13.0 & 20.2 & 10.4 & 15.6\\
PRTS w/ CRL  & 26.8 & 62.2 & 36.0 & 8.8 & 19.0 & 33.8 & 15.4 & 21.2\\
\midrule
\textbf{$\Delta$} & 0.0 & \up{+27.8} & \up{+31.2} & \down{-1.2} & \up{+6.0} & \up{+13.6} & \up{+5.0} & \up{+5.6}\\
\bottomrule
\end{tabular}
}
\caption{\textbf{Suite-level performance breakdown on LIBERO-Pro.}
Success rates (\%) are decomposed by the original LIBERO suites to illustrate performance under \textit{Position} (spatial adaptability) and \textit{Task} (procedural generalization) distribution shifts. 
The $\Delta$ denotes PRTS w/ CRL minus PRTS w/o CRL.}
\label{tab:crl_ablation_libero_pro_detail}
\end{threeparttable}
\end{table}

While the LIBERO-Plus results demonstrate that CRL provides robustness when underlying task semantics are fixed, the LIBERO-Pro evaluation reveals a deeper advantage: the capability to generalize when semantics are fundamentally altered. Because contrastive pre-training anchors state representations to goal reachability rather than specific visual trajectories, the policy avoids overfitting to nominal layouts. Consequently, PRTS w/ CRL can dynamically extrapolate learned behaviors to novel spatial and procedural configurations. This further validates that our approach equips the downstream action expert with a highly composable representation space, enabling true procedural generalization beyond mere robustness.

To further explore why CRL improves robustness and instruction-conditioned generalization, we visualize the state-action representations extracted from the pretrained backbone using t-SNE~\citep{van2008tsne}. 
Specifically, we sample trajectories from ten representative manipulation tasks in the RoboMind dataset covering diverse skills and effector. From each episode, we extract non-overlapping temporal chunks. Then for each chunk, we feed the corresponding observation, proprioceptive state, language instruction, and action tokens into the pretrained backbone, and use the hidden state at the \texttt{<|action\_end|>} token as the state-action representation since this token encodes the full multimodal context of the sequence.
These representations are extracted separately from the PRTS w/ CRL and PRTS w/o CRL models and projected into a unified t-SNE space. The obtained representations are shown in Fig.~\ref{fig:tsne_vis}.

\begin{figure}[htbp]
    \centering
    \includegraphics[width=\linewidth]{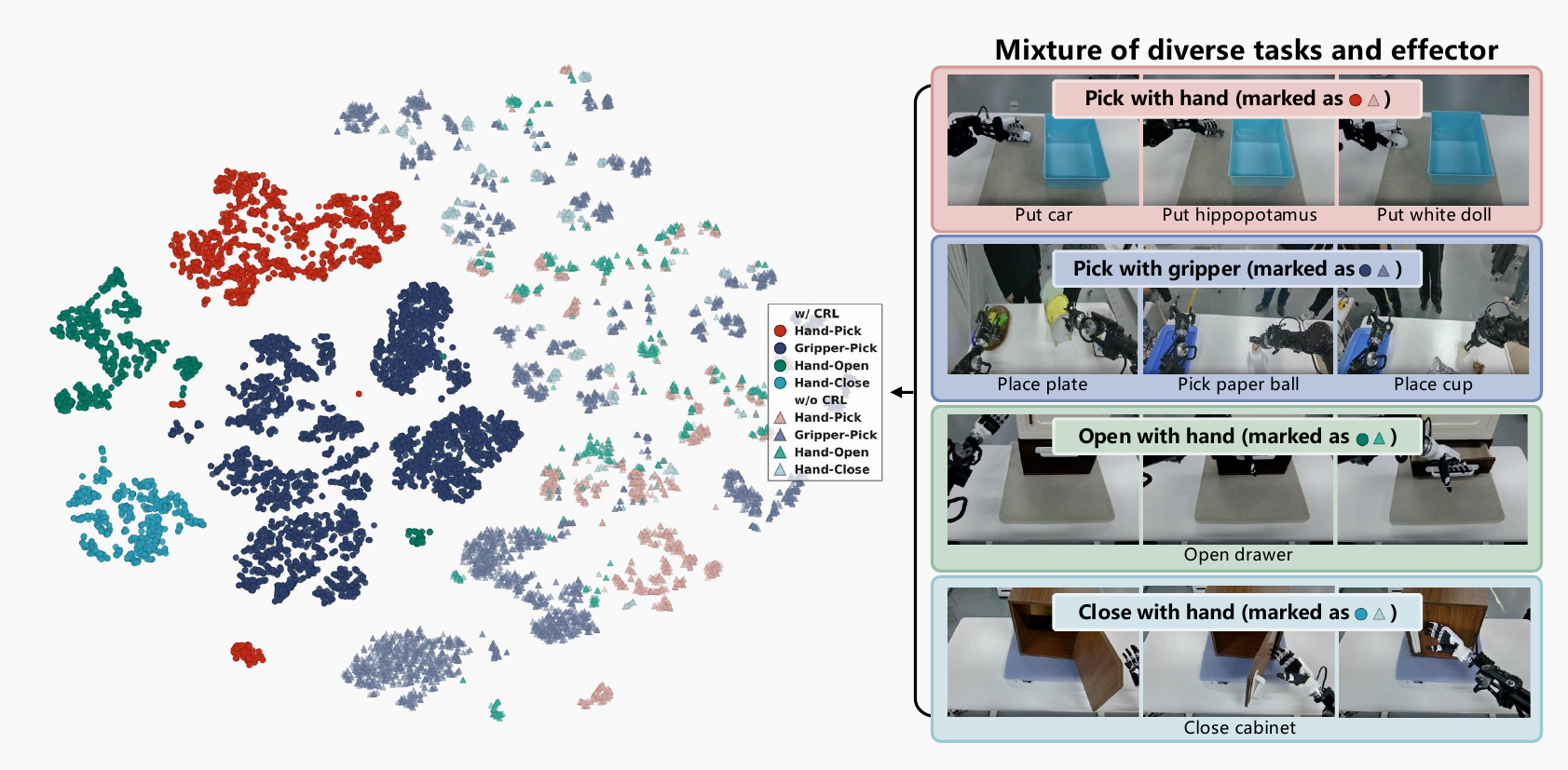}
    \caption{\textbf{T-SNE visualizations of learned task representations.}}
    \label{fig:tsne_vis}
\end{figure}

The inclusion of CRL organizes the embedding space around reusable manipulation primitives. Samples corresponding to \textit{Hand-Pick}, \textit{Gripper-Pick}, \textit{Hand-Open}, and \textit{Hand-Close} form tightly clustered, well-separated neighborhoods. This indicates that CRL successfully pulls together state-action pairs that share similar goal-reaching semantics, while pushing apart trajectories associated with divergent outcomes.
In contrast, the representations from the non-CRL variant remain highly entangled and diffuse. This suggests that pure behavioral cloning is prone to overfitting to trajectory-specific correlations, rather than capturing the shared, goal-reaching behavioral structure across demonstrations.
This structural alignment explains the performance gains across the LIBERO benchmark family. Under the perceptual shifts of LIBERO-Plus, this cleanly separated space allows the policy to robustly maintain the intended primitive. Furthermore, under the severe semantic shifts of LIBERO-Pro, this composable representation equips the policy to accurately ground novel instruction goals and dynamically synthesize required behaviors, entirely bypassing rote memorization.

\subsection{Controlled Real-World Generalization}
\label{sec:generalization}

The real-world results above evaluate whether PRTS can execute a broad set of tasks under the standard deployment setting. We next ask a stricter question: can the same policy preserve task success when the scene deviates from the post-training distribution in controlled ways? To answer Q3, we select four representative RealMan tasks, \textit{Paper Rubbish}, \textit{Place Block}, \textit{Pick Shoes} and \textit{Stack Cups}. These tasks cover bimanual receptacle interaction, precise target placement, object organization with closure and multi-object sequential stacking, respectively.

We evaluate each task under four axes of real-world generalization, visualized in Figs.~\ref{fig:realman_visual_generalization} and~\ref{fig:realman_task_generalization}. We arrange the axes from lower-level perceptual shifts to higher-level semantic shifts. \textit{Illumination} changes the observation statistics by turning off the large background light behind the robot. \textit{Spatial} moves the interaction objects to positions outside the post-training placement distribution. \textit{Object} replaces the manipulated objects with same-category instances that differ in color or size while keeping the instruction semantics fixed. Finally, \textit{Task} generalization recombines components from training tasks into new language instructions. Unlike object generalization, this axis changes the language-specified goal itself: the required skills have appeared during post-training, but they must be retrieved from different task contexts and recomposed in a new scene. These controlled perturbations serve as a physical counterpart to the distribution shifts studied in LIBERO-Plus and LIBERO-Pro, but expose the policy to real sensor, contact, and workspace changes.

\begin{figure}[htbp]
    \centering
    \includegraphics[width=0.85\linewidth]{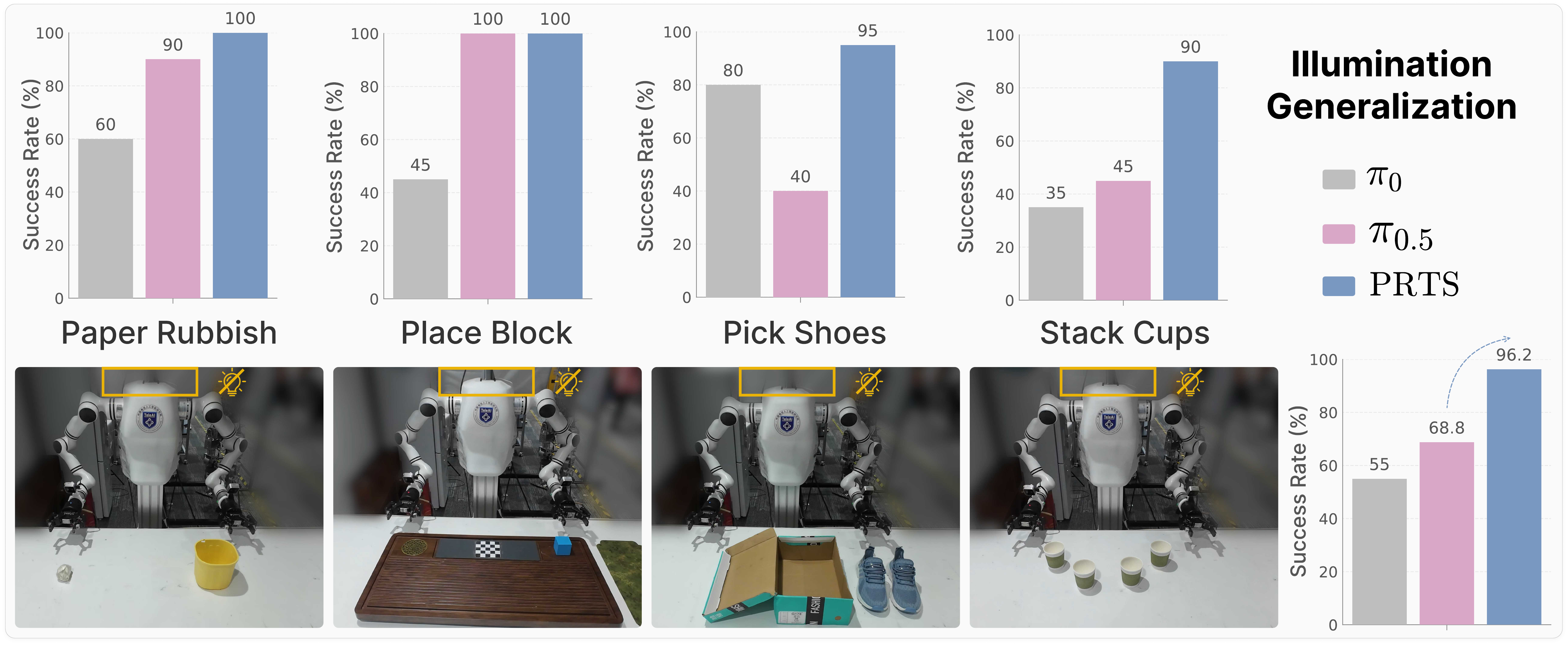}
    \vspace{0.4em}

    \includegraphics[width=0.85\linewidth]{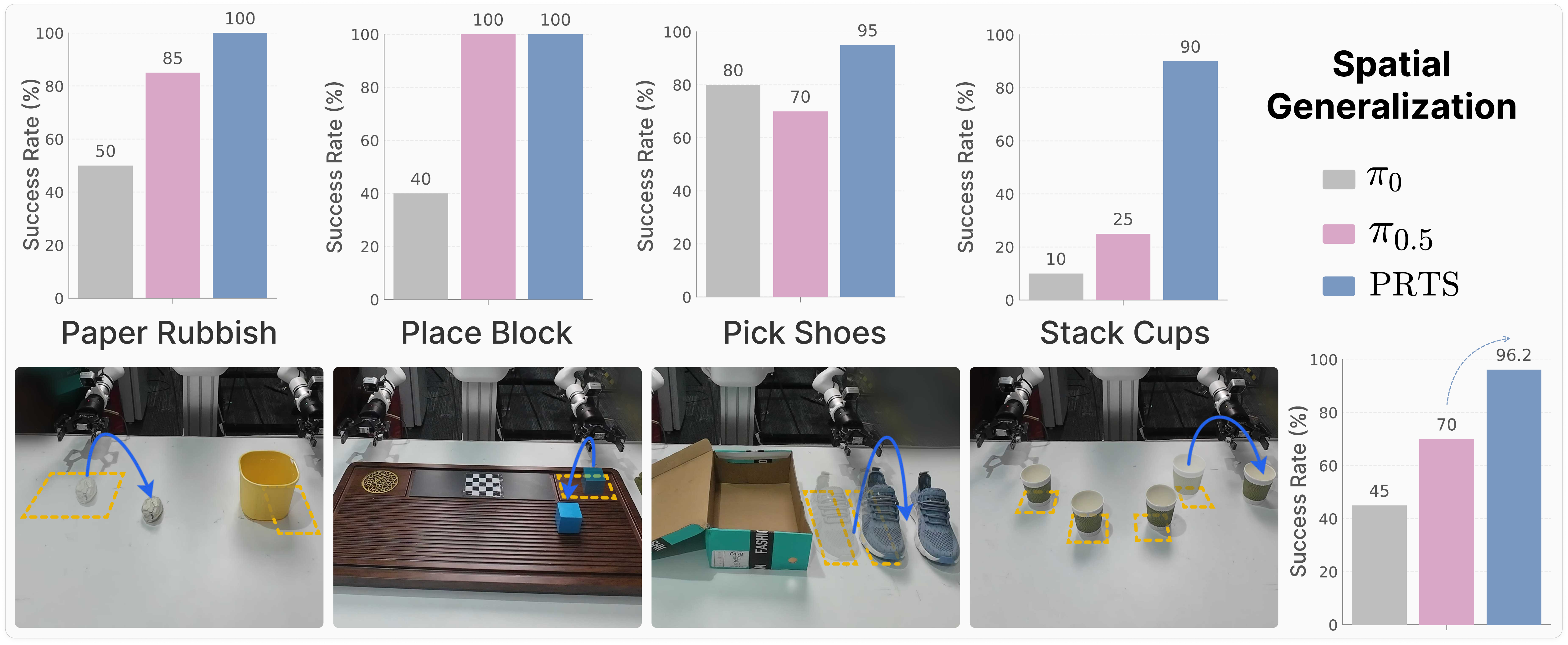}
    \vspace{0.4em}

    \includegraphics[width=0.85\linewidth]{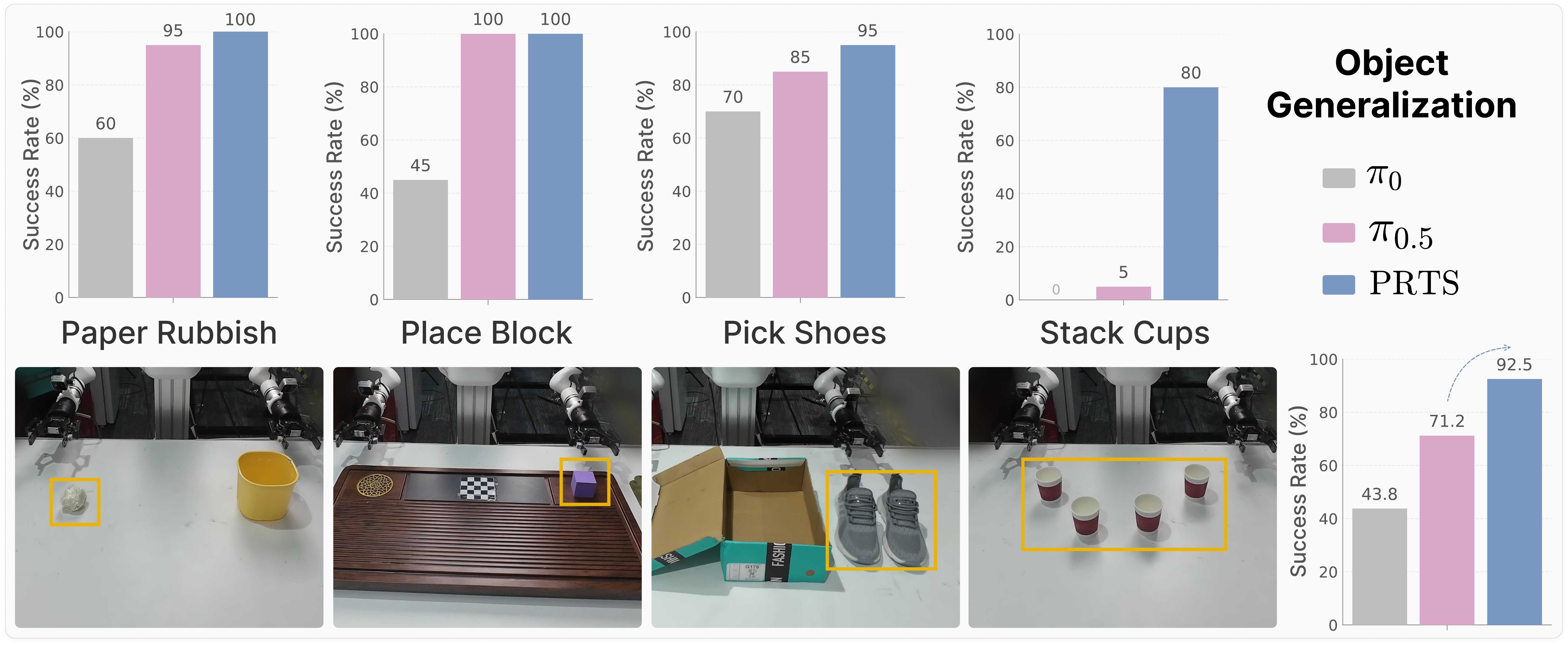}
    \caption{\textbf{Controlled visual, layout, and object shifts on the RealMan platform.} We visualize the three lower-level generalization settings used in the RealMan study. Illumination generalization changes the lighting condition, spatial generalization moves the manipulated objects outside the post-training layout distribution, and object generalization changes the visual instance while preserving the semantic object category and the task instruction. Quantitative results are reported in Table~\ref{tab:realman_generalization}.}
    \label{fig:realman_visual_generalization}
\end{figure}

\begin{figure}[htbp]
    \centering
    \includegraphics[width=0.85\linewidth]{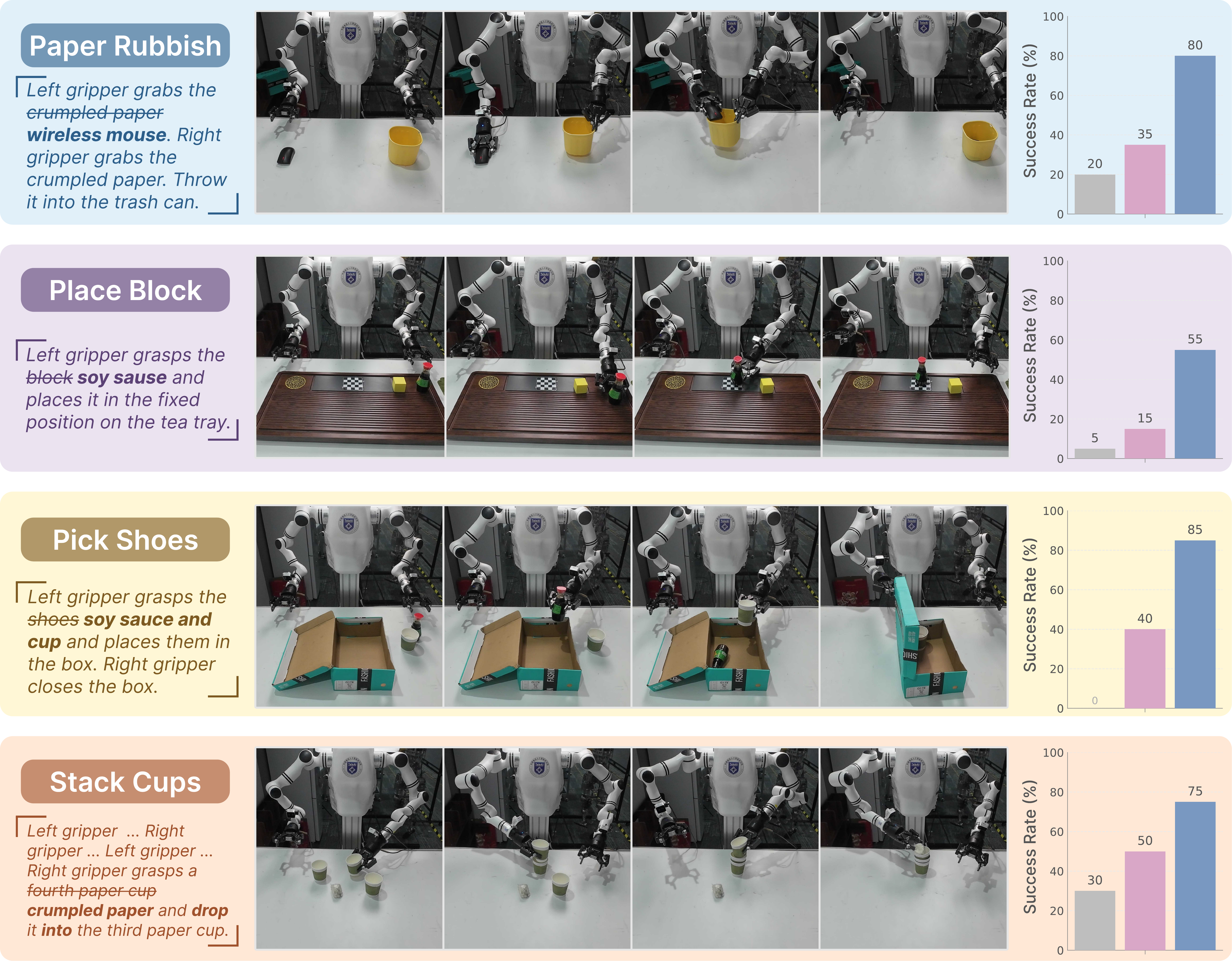}
    \caption{\textbf{Controlled task-level generalization on the RealMan platform.} Task generalization changes the language goal by recombining components of seen tasks, making it semantically stricter than illumination, spatial, or object-level shifts. This setting tests whether the policy can bind a novel instruction to feasible manipulation steps instead of merely recognizing familiar objects or replaying a familiar trajectory. Quantitative results are reported in Table~\ref{tab:realman_generalization}.}
    \label{fig:realman_task_generalization}
\end{figure}

Table~\ref{tab:realman_generalization} shows that PRTS is highly robust to visual, layout, and object shifts in the real world. The advantage is consistent across task types: \textit{Paper Rubbish} and \textit{Place Block} remain at $100\%$ success across these three axes, while \textit{Pick Shoes} and \textit{Stack Cups} stay at least $80\%$ despite changes in object appearance and placement. This suggests that PRTS is not simply replaying trajectories tied to a narrow scene configuration, but can preserve the intended manipulation behavior under substantial visual and layout variation.

Task generalization is more challenging because the language instruction itself changes. This is where the gap to prior VLAs becomes largest: PRTS averages $73.8\%$ success on the task-instruction axis, while $\pi_{0.5}$ reaches $35.0\%$ and $\pi_0$ reaches only $13.8\%$. Importantly, these are not arbitrary unseen skills. Each task recombines primitives that appeared during post-training, but places them under a new instruction and scene context. In \textit{Paper Rubbish}, the original trash-can grasp is replaced by grasping a flat wireless mouse from the long-horizon office task; this requires a qualitatively different, deeper and more planar grasp than the crumpled-paper grasp. Prior models tend to execute the original paper-rubbish grasping pattern and fail to stably grasp the mouse, while PRTS transfers the mouse-grasping skill into the new task and reaches $80\%$ success. In \textit{Place Block}, the target object is changed to a soy-sauce bottle that must be grasped at a narrow stable region, while both the bottle and the original block are present. This tests whether the policy follows the language-specified object rather than the scene's habitual block action; PRTS reaches $55\%$, whereas $\pi_0$ mostly continues to grasp the block and $\pi_{0.5}$ often selects the bottle but drops it due to poor grasp placement.

The other two task-generalization cases further test whether the model can preserve task structure while replacing key objects or substeps. In \textit{Pick Shoes}, the two shoes are replaced by a soy-sauce bottle and a paper cup, producing a visually different scene and requiring object-specific grasps before closing the box. $\pi_0$ often moves to where the shoes would have been and stalls, and $\pi_{0.5}$ can pick objects but frequently fails to place them into the box. PRTS instead completes the recomposed sequence with $85\%$ success. In \textit{Stack Cups}, the final paper-cup stacking step is replaced by dropping a crumpled paper ball into the third cup. This removes a key visual cue used for localizing the next stacking target and requires the policy to track the current execution stage precisely. PRTS reaches $75\%$, while prior models show larger precision degradation even in earlier cup-stacking stages. Overall, these results support our central claim that goal-reachability-aware pre-training yields a real-world policy that better preserves the connection between language goals and feasible state-action outcomes under distribution shifts, while also revealing a clear limitation: instruction-level recombination remains one of the hardest forms of real-world generalization.

\subsection{Robustness and Recovery under Human Interventions}
\label{sec:robustness}

The generalization study above changes the initial scene before execution begins. We next consider a more dynamic question: can the policy recover when the world changes \emph{during} execution? This setting is important because real deployments rarely proceed as a clean open-loop rollout. Objects may be moved by people, a grasped item may be taken away and reset, and a nearly completed task may be forced back to an earlier state. To answer Q4, we conduct intervention stress tests on four representative tasks: \textit{Paper Rubbish}, \textit{Pick Shoes}, and \textit{Place Block} on the RealMan dual-arm platform, and \textit{Gear Assembly} on the Flexiv platform.

\begin{figure}[htbp]
    \centering
    \includegraphics[width=0.85\linewidth]{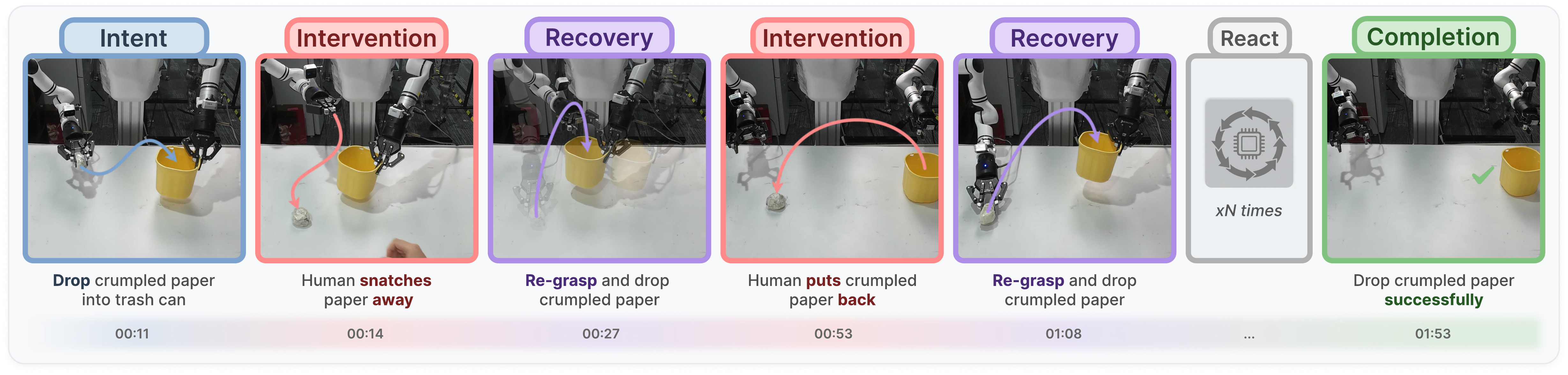}
    \vspace{0.35em}

    \includegraphics[width=0.85\linewidth]{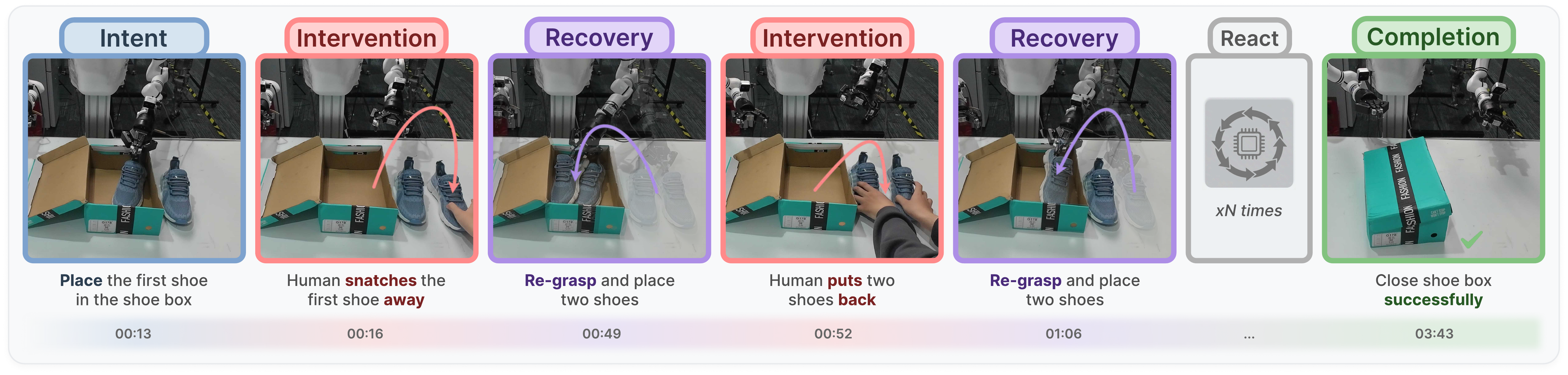}
    \vspace{0.35em}

    \includegraphics[width=0.85\linewidth]{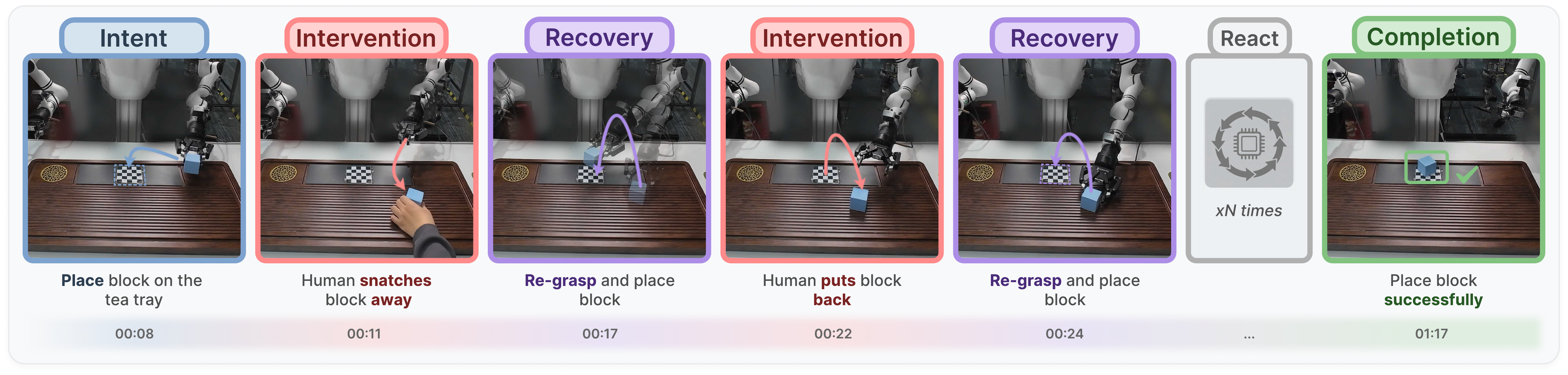}
    \vspace{0.35em}

    \includegraphics[width=0.85\linewidth]{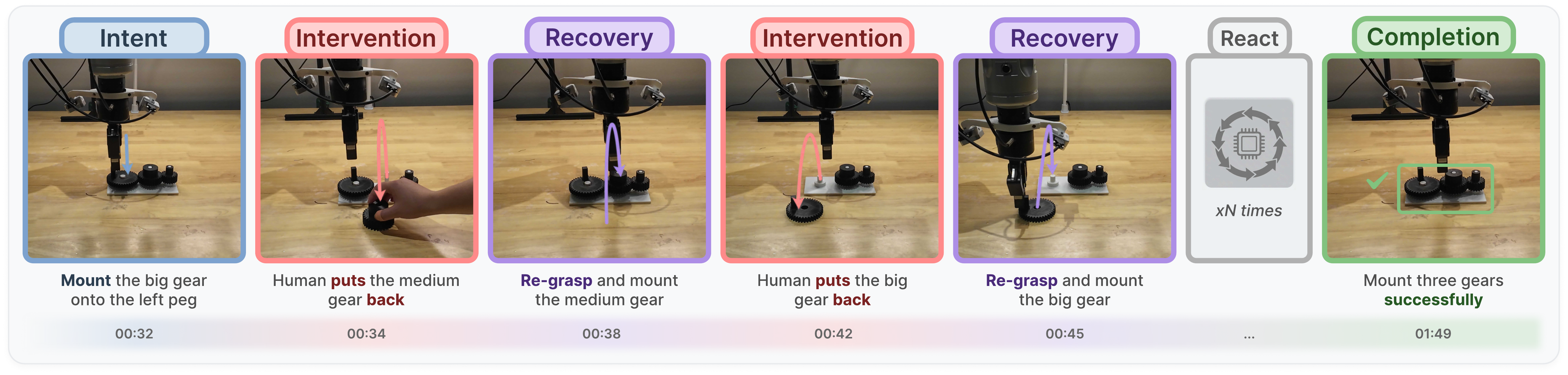}
    \caption{\textbf{Robustness and recovery under human interventions.} We visualize representative intervention rollouts in which a human repeatedly perturbs task-relevant objects after the policy has committed to an action or subgoal. The traces cover bimanual receptacle interaction, multi-object organization, precise placement, and contact-rich assembly. Successful recovery requires the policy to recognize the changed state, re-estimate task progress, and choose the next action according to the still-unreached language goal.}
    \label{fig:human_intervention}
\end{figure}

Figure~\ref{fig:human_intervention} visualizes the intervention protocol. Rather than only perturbing the initial condition, a human interrupts the rollout after the robot has already formed an intent or started executing a substep. The perturbation can invalidate the current action, reset an object to an earlier state, or change the progress of a multi-stage task. We therefore treat recovery as more than retrying a failed grasp: the policy must observe the new state, abandon stale assumptions about what has already been completed, and continue toward the original instruction. These trials are intended as qualitative stress tests rather than a fully instrumented benchmark, so we report the observed recovery behavior instead of a formal success-rate table.

Across roughly five intervention trials per task, PRTS consistently recovered from human perturbations and completed the task. \textit{Place Block} provides a controlled example with three interruptions: moving the block before grasping, removing it during transport and resetting it to the grasping area, and moving it back after successful placement while the arm is returning home. PRTS responds by tracking the moving target, stopping stale actions when the block is removed, or re-entering the full task procedure after post-completion reset, rather than treating the task as already finished.

The other tasks test the same recovery principle in different task structures. \textit{Paper Rubbish} and \textit{Pick Shoes} use similar object-removal and reset interventions, but require recovery in bimanual receptacle interaction and multi-object organization. The policy must re-localize the manipulated objects, infer which subgoals remain unfinished, and redo only the affected stages before completing the original instruction. \textit{Gear Assembly} uses a heavy intervention: after the assembly sequence has been completed, human operators move one of the three gears back to its original position. PRTS still detects the changed state and executes the corresponding gear grasping and placement again, showing that recovery remains goal-conditioned even after apparent task completion and in a contact-rich setting.

The baseline failures reveal why this setting is stricter than standard real-world evaluation. $\pi_0$ largely ignores the changed state and continues replaying fixed task phases, so all three \textit{Place Block} perturbations disrupt its execution. $\pi_{0.5}$ is more reactive under simple interruptions, such as tracking a moved object before grasping or retrying the current substep. However, under stronger interventions that revert a nearly completed or completed task to an earlier stage, it often treats tasks as completed and fails to re-enter the necessary earlier subgoals; this failure is especially clear in the post-completion \textit{Place Block} reset and the heavy \textit{Gear Assembly} intervention. In contrast, PRTS repeatedly re-selects actions based on the currently reachable path to the language goal. This supports the central role of goal-reachability-aware representations: robustness under intervention requires not only recognizing objects, but also judging which goal-conditioned state-action outcomes remain feasible after the execution history has been invalidated.

\begin{figure}[htbp]
    \centering
    \includegraphics[width=\linewidth]{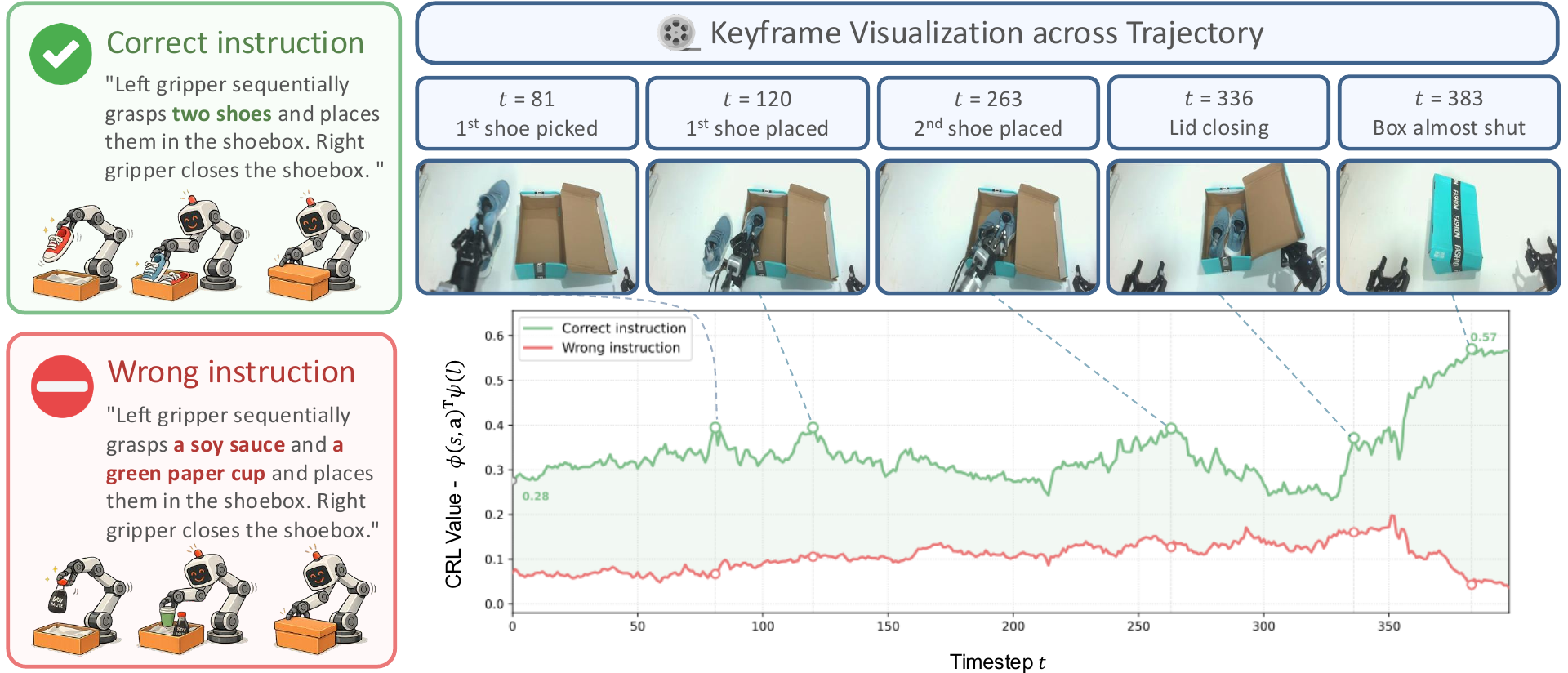}
    \caption{\textbf{CRL value on a held-out \textit{Pick Shoes} trajectory, scored by the pre-trained PRTS with no RealMan task suite fine-tuning.}
    \emph{Left:} the two instructions used to score the same demonstration---a correct one (green, top) that describes what the trajectory actually does, and a wrong one (red, bottom) that swaps the named grasp targets to objects (soy sauce, green paper cup) that never appear in the scene.
    \emph{Top right:} five keyframes from the trajectory with their timestep indices, marking the structural milestones of the task (first shoe lifted, first shoe placed, second shoe placed, lid closing, box almost shut).
    \emph{Bottom right:} the contrastive value $\phi(s_t,\mathbf{a}_t)^{\top}\psi(l)$ as a function of timestep $t$, computed under the correct instruction (green) and the wrong instruction (red) on the same state-action sequence; dashed lines link each keyframe to its corresponding marker on the green curve.
    The green curve forms progress-sensitive local peaks at the five keyframes and rises sharply to $0.57$ as the shoebox is closed, while the red curve stays in $0.04\, \sim \,0.20$  throughout and never crosses the green curve in any frames.}
    \label{fig:value_vis}
\end{figure}

\subsection{CRL Value Visualization}\label{sec:value}
The previous experiments establish that CRL pre-training improves policy robustness under distribution shift. Here we directly inspect its value.
Theorem~\ref{thm:equivalence} shows that the inner product $\phi(s_t,\mathbf{a}_t)^{\top}\psi(l)$ tracks the log-transformed goal-conditioned action value $\log Q^{\pi}_l(s_t,\mathbf{a}_t)$:
the score should increase as the trajectory approaches goal completion under the correct instruction, and remain low under a mismatched instruction.

To this end, we evaluate the CRL encoder heads on the \textbf{pre-traiend PRTS} before any post-training on RealMan task suite as we mention in Sec.~\ref{sec:benchmark}.
Demonstrations from the eleven RealMan tasks are therefore out-of-distribution for both the backbone and the contrastive heads, which directly tests whether the CRL-shaped representation captures goal-reachability semantics on entirely new trajectories without any post-training adaptation.
We randomly select a held-out \textit{Pick Shoes} demonstration trajectory and score it under two instructions on the \emph{same} state-action sequence $(s_t,\mathbf{a}_t)$: a correct instruction that describes the actual task (\textit{grasp two shoes and place them into the shoebox}), and a deliberately mismatched instruction that replaces the grasp targets with absent objects (\textit{soy sauce} and \textit{green paper cup}), which is exactly the instruction edit used in the novel-task generalization study (Sec.~\ref{sec:generalization}).
A faithful CRL value should assign consistently higher scores to the correct instruction throughout the episode.

Fig.~\ref{fig:value_vis} supports this behavior clearly. Along the \textit{Pick Shoes} trajectory, the value under the correct instruction shows a clear overall upward trend as the rollout progresses toward task completion, with local peaks consistently aligning with key manipulation milestones---firmly grasping the first shoe ($t{=}81$), placing it into the shoebox ($t{=}120$), placing the second shoe ($t{=}263$), and finally closing the lid ($t{=}336$ and $t{=}383$).
In particular, the score rises from $0.28$ at the initial frame to $0.57$ as the shoebox is nearly closed, with the largest jump concentrated in the final lid-closure phase. Although the curve is not strictly monotonic at every timestep, its major increases are tightly coupled with semantically meaningful task progress rather than tracking superficial visual changes or arm motion alone.
This shows that the CRL head captures goal-reaching structure instead of merely reacting to visual changes: the model correctly assesses the goal-reaching feasibility of the current state-action pair on a per-frame basis.

In contrast, the value under the wrong instruction stays consistently low ($0.04\, \sim \,0.20$) and even drifts downward toward the end of the rollout, because the executed actions are shoe-grasping behaviors that do not move toward a soy-sauce or paper-cup goal at any frame.
Notably, the red curve never crosses the green one across all timesteps, and the shaded green band visualizes this persistent margin. 
This confirms the two key properties expected from CRL: temporal-progress-sensitive value estimation and strong goal-conditioned discriminability without any post-training adaptation.

\subsection{Pre-training Efficiency}\label{sec:efficiency}
A natural concern is whether adding CRL token blocks and role-aware masking makes large-scale VLA pre-training prohibitively expensive. We show that the answer is no: with the fused CuTe kernel, the additional cost stays close to the FlashAttention~3 (FA3) baseline at the attention level, and thus end-to-end training maintains near-linear throughput scaling up to $64$ H100 GPUs.

Sec.~\ref{sec:language_crl} introduces the role-aware attention mask for the CRL token blocks and its CuTe-kernel implementation on top of FlashAttention~\citep{zadouri2026flashattention4}, which applies the mask at block granularity. Here we quantify its efficiency from two perspectives: a kernel microbenchmark that isolates per-layer attention time under fixed packing on $1$ H100 GPU, and a throughput study that measures aggregate training tokens/s when the same per-device packing configuration is scaled across $\{2,4,8,32,64\}$ H100 GPUs using DeepSpeed ZeRO-2 distributed strategy~\citep{rajbhandari2020zero}.

\textbf{Kernel microbenchmark.}
All rows in Fig.~\ref{fig:pretrain_efficiency}a share the same setup: a local batch with size $bs = 4$ and packed sequence length $4096$. We compare (i) regular BC forward using FA3 without the two CRL token blocks, (ii--iii) FlexAttention with Triton and FlashAttention as backends implementing the joint CRL + BC forward with the role-aware mask, and (iv) our fused CuTe kernel that feeds FlashAttention a block-sparse layout for the CRL blocks.
Notably, the per-layer attention forward time of the fused kernel is only $1.18\times$ of the FA3 reference while enabling the full CRL objective optimization.
The large gap between FlexAttention and CuTe mainly comes from how the five-role mask is implemented: expressing it through a general flexible-attention path introduces substantial overhead, while integrating the sparsity pattern directly into the fused kernel keeps the cost close to standard BC training.

\textbf{Throughput scaling.}
Fig.~\ref{fig:pretrain_efficiency}b plots aggregate tokens/s against GPU count for the full training step, using $4$ packed sample per GPU and DeepSpeed ZeRO-2 strategy as a controlled sweep of the distributed training run.
Against a linear reference anchored at the $2$-GPU point, scaling efficiency remains at almost $100\%$ through $8$ GPUs, $95.2\%$ at $32$, and $85.1\%$ at $64$. The drop at larger scale matches the expected communication overhead of ZeRO-2 rather than CRL-specific computation.

\begin{figure}[htbp]
    \centering
    \includegraphics[width=\linewidth]{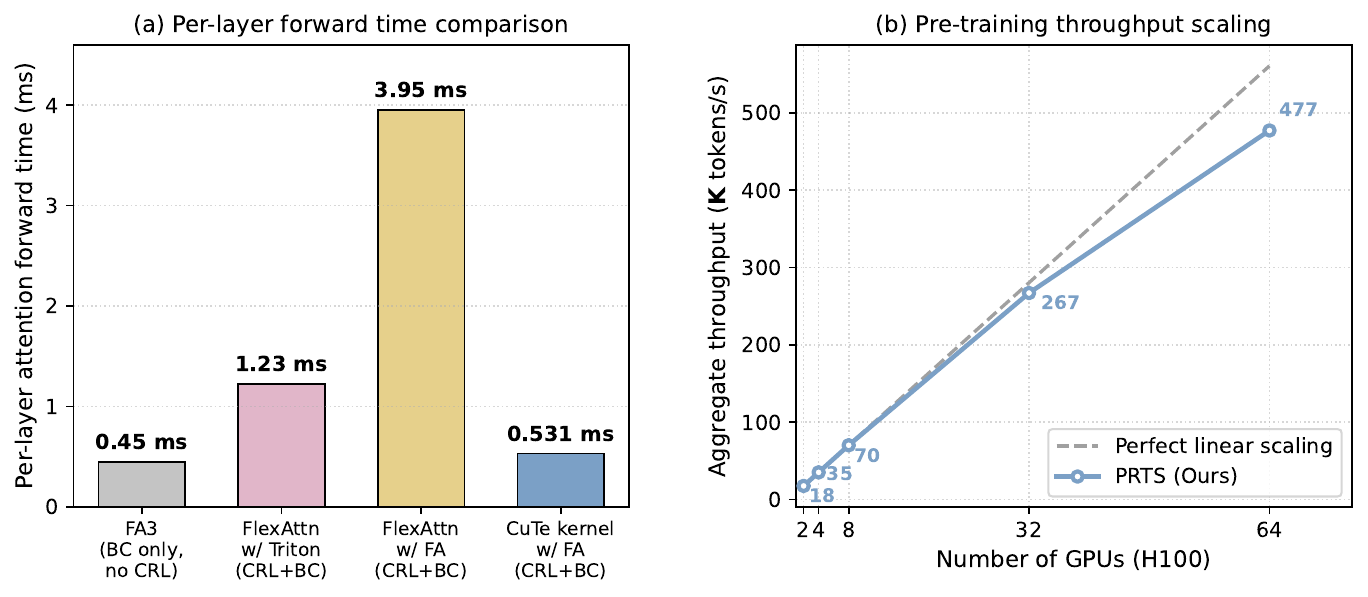}
    \caption{\textbf{Pre-training efficiency of PRTS.}
    \textbf{(a)} Per-layer attention forward time on $2$ H100 GPUs for BC-only FA3, two FlexAttention backends, and the fused CuTe kernel.
    \textbf{(b)} Aggregate token throughput vs.\ number of H100 GPUs; dashed line: linear scaling from the $2$-GPU point.}
    \label{fig:pretrain_efficiency}
\end{figure}

Together, these results show that CRL adds little overhead when implemented with the fused role-aware kernel: per-layer attention time remains close to the FA3 baseline, and end-to-end throughput scales nearly linearly up to $64$ H100 GPUs. This makes large-scale contrastive pre-training practical without sacrificing pre-training efficiency.

%% file: sec/7_conclusion.tex
\section{Conclusion}
\label{sec:conclusion}

We presented \textbf{PRTS}, a VLA foundation model that, for the first time, scales {reward-label-free} contrastive reinforcement learning into VLA pre-training itself. A language-conditioned contrastive objective shapes the backbone so that the pre-trained VLA encodes goal-reachability-aware representations, drawing supervision {entirely} from offline trajectory structure, without manual reward annotations or a separate value network (Sec.~\ref{sec:language_crl}).
A set of system-level optimizations---most notably a role-aware causal mask fused into FlashAttention via a custom CuTe kernel, together with sequence packing and sharded contrastive similarity computation---folds the contrastive head into the same forward pass as behavior cloning at near-BC throughput, making CRL pre-training feasible at the parameter scale of current VLA backbones (Sec.~\ref{sec:efficiency}).

Across LIBERO, LIBERO-Plus, LIBERO-Pro, SimplerEnv (WidowX), and real-world bimanual RealMan and single-arm Flexiv platforms, PRTS matches or exceeds the strongest prior VLAs at substantially smaller post-training compute, with the gap to baselines widening as evaluation moves further off the post-training distribution---novel instruction following, long-horizon execution, and human-intervention recovery (Sec.~\ref{sec:benchmark}--\ref{sec:robustness}). A direct examination of the value prediction further confirms that the same backbone produces a goal-reachability-aware and goal-discriminative score on out-of-distribution rollouts without any post-training adaptation (Sec.~\ref{sec:value}).

Returning to the premise that motivated this work---that robotic trajectory learning is inherently a goal-reaching process and that goal-conditioned RL has been shown to scale gracefully in pure decision-making domains~\citep{1000-layer}---these results extend that scaling story to VLA pre-training itself: scaling reward-free, language-conditioned CRL inside the VLA backbone delivers the goal-reachability awareness that not only substantially lifts overall task performance, but also gives rise to remarkably strong out-of-distribution generalization and robustness, most notably in zero-shot instruction following (highlighted in Fig.~\ref{fig:realman_task_generalization}) and recovery under human interventions (highlighted in Fig.~\ref{fig:human_intervention}).
We view this kind of goal-reachability-aware behavior as an important milestone signaling that VLA pre-training is no longer just a more-capable imitator, but an emerging foundation for general-purpose, goal-directed robotic intelligence.

%% file: sec/X_suppl.tex
\clearpage
\appendix
\section{Contributions}
\textbf{Data collection and standardization:} Jiangyuan Zhao, Yang Zhang, Fangzheng Yan

\textbf{Annotation and supplemental data:} Jiangyuan Zhao, Yang Zhang

\textbf{Training infrastructure:} Yang Zhang, Jiangyuan Zhao, Tian Li, Haitong Tang, Sen Fu, Xuan'er Wu, Qizhen Weng

\textbf{Model architecture and algorithm:} Yang Zhang

\textbf{Model pre-training:} Yang Zhang, Jiangyuan Zhao

\textbf{Model post-training and evaluation:} Yang Zhang, Jiangyuan Zhao, Chenyou Fan

\textbf{Writing and illustration:} Yang Zhang, Jiangyuan Zhao, Chenyou Fan, Chenjia Bai

\textbf{Project lead:} Yang Zhang

\textbf{Advisor:} Chenjia Bai

\section{Proof of Theorem~\ref{thm:equivalence}}\label{app:proof}

We first formalize the $l \to \text{sa}$ objective with explicit expectation notation. Let $\mathcal{D}$ denote the data distribution over trajectories $\tau$, and let $\mathcal{B} \sim \mathcal{D}^B$ denote a batch of $B$ samples. For a given language goal $l$, let $\mathcal{S}(l; \mathcal{B}) = \{(s_j, a_j, t_j, T_j) \in \mathcal{B} : \text{task}(s_j) = l\}$ be the set of state-action pairs in the batch belonging to task $l$, where $t_j$ is the timestep and $T_j$ is the trajectory length.

The $l \to \text{sa}$ loss is:
\begin{align}
    \mathcal{L}^{l \to \text{sa}}(\phi, \psi) = \mathbb{E}_{\mathcal{B} \sim \mathcal{D}^B} \left[ \frac{1}{|\mathcal{B}|} \sum_{i \in \mathcal{B}} \left( -\sum_{j \in \mathcal{S}(l_i; \mathcal{B})} q_{ij} \log \frac{\exp(\psi(l_i)^\top \phi(s_j, a_j))}{\sum_{k \in \mathcal{B}} \exp(\psi(l_i)^\top \phi(s_k, a_k))} \right) \right],
\end{align}
where the weights $q_{ij} = \frac{\gamma^{T_j - t_j}}{\sum_{j' \in \mathcal{S}(l_i; \mathcal{B})} \gamma^{T_{j'} - t_{j'}}}$ satisfy $\sum_{j \in \mathcal{S}(l_i; \mathcal{B})} q_{ij} = 1$.

\begin{proof}
Consider a fixed goal $l$ and the subset of the batch $\mathcal{S}(l; \mathcal{B}) = \{j_1, \ldots, j_m\}$ containing $m$ samples from this task. For this fixed $l$, the loss contribution is:
\begin{align}
    \mathcal{L}_l = -\sum_{r=1}^{m} q_{r} \log p_r,
\end{align}
where $p_r = \frac{\exp(\psi(l)^\top \phi(s_{j_r}, a_{j_r}))}{\sum_{k \in \mathcal{B}} \exp(\psi(l)^\top \phi(s_k, a_k))}$ and $q_r = \frac{\gamma^{T_{j_r} - t_{j_r}}}{\sum_{r'} \gamma^{T_{j_{r'}} - t_{j_{r'}}}}$.

The cross-entropy $\mathcal{L}_l$ is minimized if and only if $p_r = q_r$ for all $r \in \{1, \ldots, m\}$. Taking logarithms:
\begin{align}
    \psi(l)^\top \phi(s_{j_r}, a_{j_r}) - \log Z = \log q_r = (T_{j_r} - t_{j_r}) \log \gamma - \log \left(\sum_{r'} \gamma^{T_{j_{r'}} - t_{j_{r'}}}\right),
\end{align}
where $Z = \sum_{k \in \mathcal{B}} \exp(\psi(l)^\top \phi(s_k, a_k))$ is the partition function (constant with respect to $r$).

Rearranging:
\begin{align}
    \psi(l)^\top \phi(s_{j_r}, a_{j_r}) = (T_{j_r} - t_{j_r}) \log \gamma + \underbrace{\left(\log Z - \log \sum_{r'} \gamma^{T_{j_{r'}} - t_{j_{r'}}}\right)}_{\triangleq C(l, \mathcal{B})}.
\end{align}

Now we establish the connection to the $Q$-function. In deterministic expert demonstrations, the policy $\pi^*$ deterministically transitions toward the goal. The discounted occupancy measure (Definition in Section Preliminaries) becomes:
\begin{align}
    Q^{\pi^*}_l(s_t, a_t) &= (1-\gamma) \sum_{k=0}^{\infty} \gamma^k \mathbb{I}[s_{t+k} = s_g] \\
    &= (1-\gamma) \gamma^{T-t},
\end{align}
where the second equality holds because under deterministic demonstrations, the goal is reached exactly at step $T$ (so $\mathbb{I}[s_{t+(T-t)} = s_g] = 1$ and all other terms are 0).

Therefore:
\begin{align}
    (T-t) \log \gamma = \log Q^{\pi^*}_l(s_t, a_t) - \log(1-\gamma).
\end{align}

Substituting back:
\begin{align}
    \psi(l)^\top \phi(s, a) = \log Q^{\pi^*}_l(s, a) + \underbrace{\left(C(l, \mathcal{B}) - \log(1-\gamma)\right)}_{\triangleq C(l)}.
\end{align}

Taking expectation over the batch sampling $\mathcal{B}$, the batch-dependent constant $C(l, \mathcal{B})$ converges to a deterministic function $C(l)$ of the goal $l$ alone, yielding the desired result.
\end{proof}

\paragraph{Comparison with Standard CRL.}
Standard CRL \citep{eysenbach2022contrastiverl} achieves $f^*(s,a,g) = \log \frac{p^{\pi}_{\gamma}(s_{t+}=g|s,a)}{p(g)} + c(s,a)$ by geometrically sampling future states as goals. This creates an \emph{implicit} weighting: states $k$ steps ahead are sampled with probability $(1-\gamma)\gamma^{k-1}$, naturally emphasizing near-term futures.

Our temporal weighting scheme achieves the \emph{same} optimal solution but through \emph{explicit} soft labels $q_{ij} \propto \gamma^{T-t}$. The key differences are:
\begin{itemize}
    \item \textbf{Sampling vs. Weighting:} CRL samples one positive per anchor (geometrically); we weight all available positives in the batch.
    \item \textbf{Goal space:} CRL requires a continuous, sampleable goal space (future states); ours works with discrete, fixed goals (language instructions).
    \item \textbf{Variance:} Our multi-positive approach reduces variance by utilizing all positive samples in the batch, rather than a single geometric sample.
\end{itemize}

Both approaches drive the representation inner product toward $\log p^{\pi}_{\gamma}(s_{t+}=g|s,a)$, confirming that temporal weighting is the appropriate adaptation of CRL to the language-goal setting.

\section{Detailed real-world task descriptions and illustrations}\label{app:task_details}

This appendix provides the natural-language task definitions and rollout illustrations for the real-world evaluation suite. The robot platform configurations and the trial-level evaluation protocol are described in Sec.~\ref{sec:setup}; here we focus on the task content, object arrangements, and the manipulation behavior required by each instruction.

\begin{figure}[htbp]
    \centering
    \IfFileExists{imgs/RealMan30F_Real_Tasks.pdf}{
        \includegraphics[width=\linewidth]{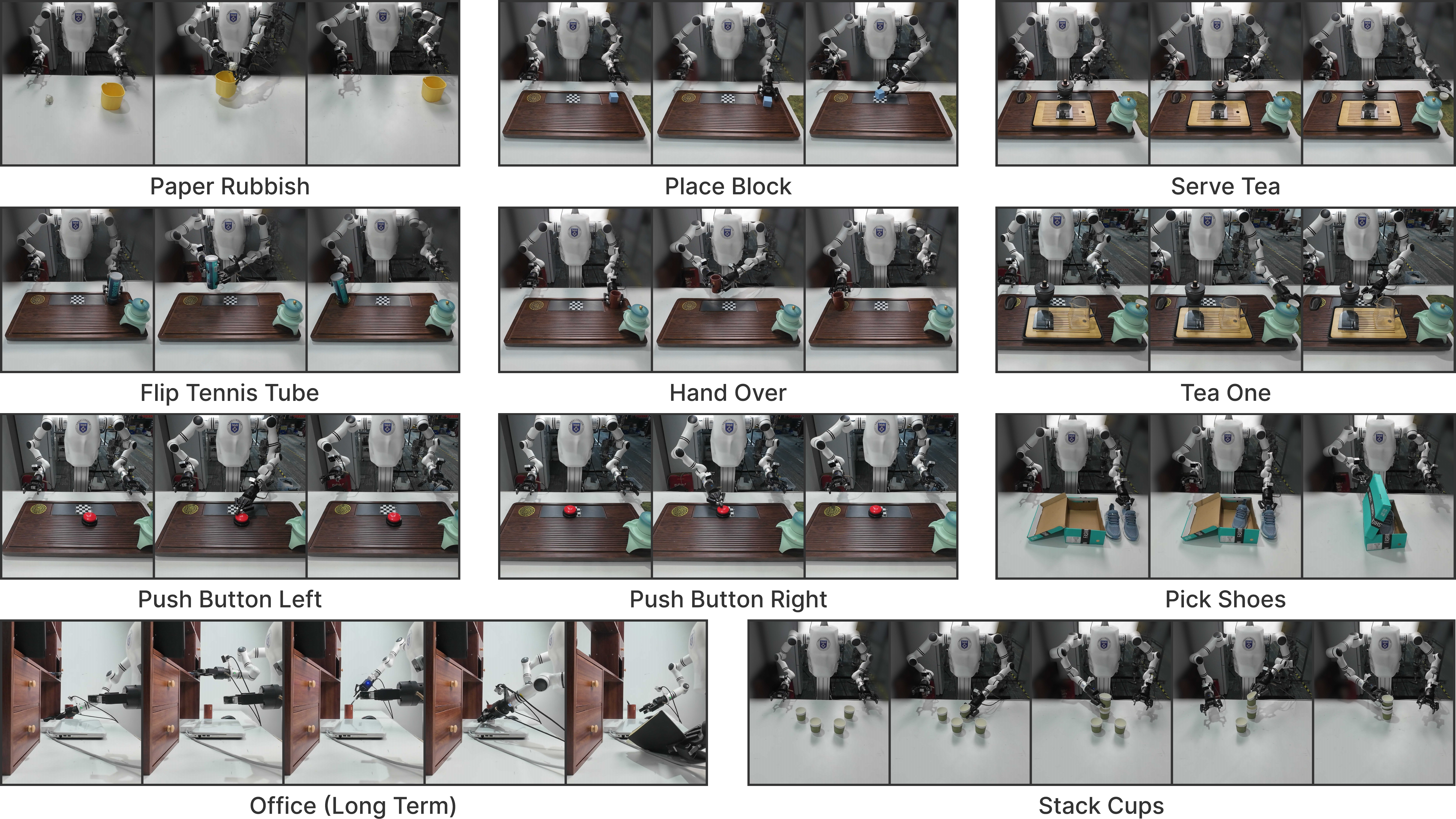}
    }{
        \fbox{\parbox{0.92\linewidth}{\centering Placeholder for \texttt{imgs/RealMan30F\_Real\_Task.pdf}.}}
    }
    \caption{\textbf{RealMan task suite.} We evaluate PRTS on eleven RealMan tasks, including ten short-horizon desktop manipulation tasks and one long-horizon office task. The montage visualizes representative rollouts for the task settings described in Appendix~\ref{app:task_details}.}
    \label{fig:realman_task_details}
\end{figure}

\subsection{RealMan dual-arm tasks}

The RealMan suite is designed to test bimanual coordination, instruction following, fine-grained contact, object transfer, and long-horizon sequencing in tabletop and office-like scenes.

\textbf{Paper Rubbish.}
The scene contains a trash can and a piece of crumpled paper. The left gripper holds the trash can while the right gripper grasps the paper and drops it into the can.

\textbf{Place Block.}
The robot grasps a block with the left gripper and places it at the specified location on the tea tray.

\textbf{Serve Tea.}
The robot grasps a teacup with the left gripper and places it on the green coaster.

\textbf{Flip Tennis Tube.}
The task begins with a tennis-ball tube on the workspace. The left gripper first grasps the tube, after which the right gripper takes the tube and places it at the designated position on the tray.

\textbf{Hand Over.}
The task requires a bimanual handover of a pen holder. The left gripper grasps the pen holder, and the right gripper receives it and places it at the fixed target position on the tea tray.

\textbf{Tea One.}
The robot grasps the white teacup on the coaster with the left gripper, places it at the fixed target position on the tea table, and then returns the left gripper to an initial-like pose.

\textbf{Push Button Left.}
The robot uses the left gripper to press the target button.

\textbf{Push Button Right.}
The robot uses the right gripper to press the target button.

\textbf{Pick Shoes.}
The robot organizes a pair of shoes into a shoebox. The left gripper sequentially picks up two shoes and places them inside the box, and the right gripper then closes the shoebox.

\textbf{Stack Cups.}
The robot stacks four paper cups into a vertical stack through alternating bimanual actions. The left gripper places the first cup on the table, the right gripper stacks the second cup on top, the left gripper stacks the third cup, and the right gripper completes the stack with the fourth cup.

\textbf{Office (Long Term).}
This task is a long-horizon office organization sequence with multiple objects and interaction types. The robot first picks up a wireless mouse and places it on the desk, then picks up a pen from the bookshelf with the right gripper, inserts it vertically into the pen holder, and releases it. It then presses the power strip with the left gripper, presses the switch with the right gripper, takes a book from the bookshelf with the left gripper, and places the book on the left side of the desk.

\begin{figure}[htbp]
    \centering
    \IfFileExists{imgs/Flexiv_Real_Tasks.pdf}{
        \includegraphics[width=\linewidth]{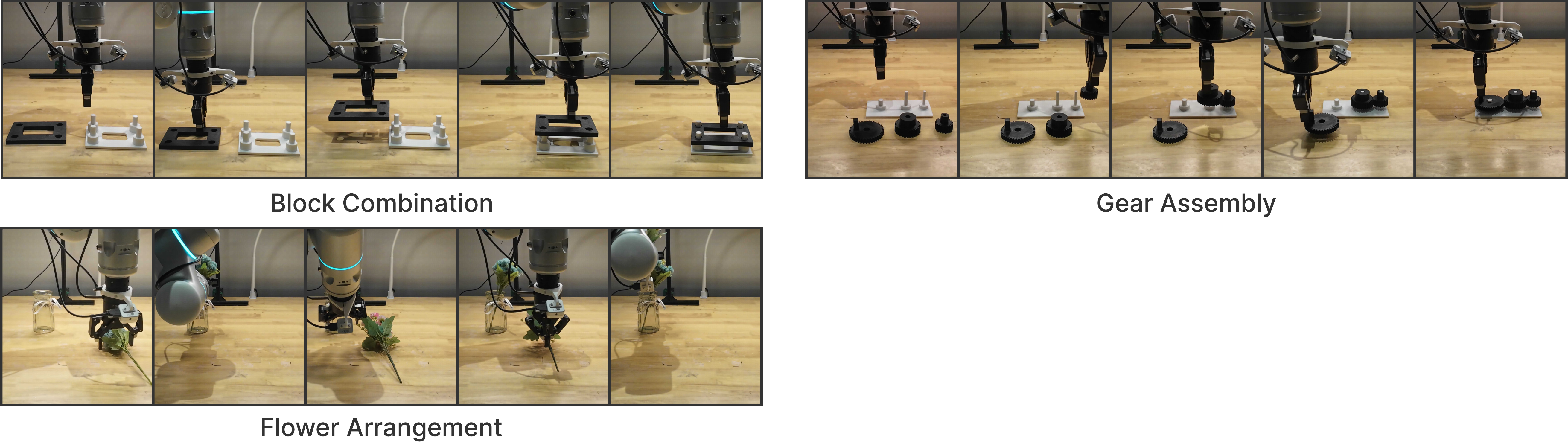}
    }{
        \fbox{\parbox{0.92\linewidth}{\centering Placeholder for \texttt{imgs/Flexiv\_Real\_Task.pdf}.}}
    }
    \caption{\textbf{Flexiv task suite.} The Flexiv evaluation complements the RealMan suite with single-arm manipulation tasks that emphasize precise alignment, assembly, and insertion-like object placement.}
    \label{fig:flexiv_task_details}
\end{figure}

\subsection{Flexiv single-arm tasks}

The Flexiv suite contains three single-arm tasks that require accurate geometric placement and contact-rich manipulation.

\textbf{Block Combination.}
The robot picks up the black square block with four corner holes and places it onto the white four-post alignment fixture, requiring the plate holes to align with the fixture posts.

\textbf{Gear Assembly.}
The robot sequentially mounts a small, medium, and big black gear onto the left, middle, and right posts of a white fixture, keeping all center holes perfectly aligned.

\textbf{Flower Arrangement.}
The robot executes a sequential pick-and-place operation, retrieving the blue and red flowers from the work surface via their stems and depositing them into the glass receptacle.

\begin{table}[htbp]
\centering
\renewcommand{\arraystretch}{1.1} 
\begin{threeparttable}

\begin{tabular}{l c c c}
\toprule
\textbf{Task} &  \textbf{$\pi_0$} & \textbf{$\pi_{0.5}$} & \textbf{PRTS} \\
\midrule
Flip Tennis Tube            & 30.0 & 85.0 & 90.0\\
Hand Over                   & 45.0 & 80.0 & 95.0\\
Office Long Term            & 5.0 & 40.0 & 95.0\\
Push Button Left            & 95.0 & 100.0 & 100.0\\
Push Button Right           & 100.0 & 100.0 & 100.0\\
Paper Rubbish               & 65.0 & 90.0 & 100.0\\
Place Block                 & 90.0 & 100.0 & 100.0\\
Pick Shoes                  & 90.0 & 90.0 & 95.0\\
Stack Cups                  & 45.0 & 70.0 & 90.0\\
Serve Tea                   & 90.0 & 90.0 & 95.0\\
Tea One                     & 85.0 & 95.0 & 95.0\\
\midrule
OVERALL                     & 67.3 & 85.5 & 95.9\\
\bottomrule
\end{tabular}

\caption{\textbf{Real-world evaluation on the RealMan dual-arm platform.} Per-task success rates over $20$ real-world trials. PRTS achieves $95.9\%$ average SR and reaches at least $90\%$ on all $11$ tasks, outperforming $\pi_{0.5}$ and $\pi_0$ under the same evaluation protocol.}

\end{threeparttable}
\end{table}

\begin{table}[htbp]
\centering
\renewcommand{\arraystretch}{1.2} 
\begin{threeparttable}

\resizebox{\textwidth}{!}{ 
\begin{tabular}{l c c c c c c c c c c c c c c c c}
\toprule
\multirow{2}{*}{\textbf{Method}} & \multicolumn{4}{c}{\textbf{Paper Rubbish}} & \multicolumn{4}{c}{\textbf{Place Block}} & \multicolumn{4}{c}{\textbf{Pick Shoes}} & \multicolumn{4}{c}{\textbf{Stack Cups}}\\
\cmidrule(lr){2-5} \cmidrule(lr){6-9} \cmidrule(lr){10-13} \cmidrule(lr){14-17}
& Light & Pos & Obj & Task & Light & Pos & Obj & Task & Light & Pos & Obj & Task & Light & Pos & Obj & Task \\
\midrule

$\pi_{0}$    & 60.0 & 50.0 & 60.0 & 20.0 & 45.0 & 40.0 & 45.0 & 5.0 & 80.0 & 80.0 & 70.0 & 0.0 & 35.0 & 10.0 & 0.0 & 30.0 \\

$\pi_{0.5}$    & 90.0 & 85.0 & 95.0 & 35.0 & 100.0 & 100.0 & 100.0 & 15.0 & 40.0 & 70.0 & 85.0 & 40.0 & 45.0 & 25.0 & 5.0 & 50.0 \\

\midrule 
\textbf{PRTS (Ours)}    & 100.0 & 100.0 & 100.0 & 80.0 & 100.0 & 100.0 & 100.0 & 55.0 & 95.0 & 95.0 & 95.0 & 85.0 & 90.0 & 90.0 & 80.0 & 75.0 \\

\bottomrule
\end{tabular}
}

\caption{\textbf{Controlled real-world generalization on the RealMan dual-arm platform.} We report success rates under illumination (\textit{Light}), spatial position (\textit{Pos}), object instance (\textit{Obj}), and task-instruction (\textit{Task}) shifts for four representative RealMan tasks. The task-instruction axis is the most semantically demanding since the language goal changes rather than only the visual scene.}

\label{tab:realman_generalization}

\end{threeparttable}
\end{table}

\newpage
\section{Detailed results on LIBERO-Pro Task and Position variations}\label{app:libero_pro_detail}

\begin{table}[htbp]
\centering
\renewcommand{\arraystretch}{1.0} 
\begin{threeparttable}

\resizebox{\textwidth}{!}{ 
\begin{tabular}{l c c c c c c c c}
\toprule
\multirow{2}{*}{\textbf{Method}} & \multicolumn{2}{c}{\textbf{libero-spatial}} & \multicolumn{2}{c}{\textbf{libero-object}} & \multicolumn{2}{c}{\textbf{libero-goal}} & \multicolumn{2}{c}{\textbf{libero-long}}\\
\cmidrule(lr){2-3} \cmidrule(lr){4-5} \cmidrule(lr){6-7} \cmidrule(lr){8-9}
& Pos (Avg.) & Task (Avg.) & Pos (Avg.) & Task (Avg.) & Pos (Avg.) & Task (Avg.) & Pos (Avg.) & Task (Avg.) \\
\midrule
OpenVLA~\citep{kim2024openvla}    & 0.0 & 0.0 & 0.0 & 0.0 & 0.0 & 0.0 & 0.0 & 0.0\\
Qwen3-VL-PI~\citep{starvla2025} {\scriptsize \color{gray} (bs=32, 30K steps)} & 14.4 & 0.0 & 0.2 & 0.0 & 1.2 & 8.8 & 1.2 & 6.8\\
ABot-M0~\citep{yang2026abotm0} {\scriptsize \color{gray} (bs=32, 30K steps)} & 13.4 & \underline{54.8} & 9.0 & \underline{9.8} & 5.6 & 10.6 & 0.2 & \underline{14.0}\\
$\pi_0$~\citep{black2024pi0} {\scriptsize \color{gray} (bs=32, 30K steps)} & 0.0 & 0.0 & 0.0 & 0.0 & 0.0 & 0.0 & 0.0 & 0.0\\
$\pi_{0.5}$~\citep{black2025pi05} {\scriptsize \color{gray} (bs=256, 30K steps)} & \underline{20.0} & 1.0 & 17.0 & 1.0 & \textbf{38.0} & 0.0 & \underline{8.0} & 1.0\\
\midrule
\textbf{PRTS (Ours)} {\scriptsize \color{gray} \textbf{(bs=32, 30K steps)}} & \textbf{26.8} & \textbf{62.2} & \textbf{36.0} & 8.8 & 19.0 & \textbf{33.8} & \textbf{15.4} & \textbf{21.2}\\
\bottomrule
\end{tabular}
}
\caption{\textbf{Detailed performance comparison on the \emph{Position} and \emph{Task} variations of LIBERO-Pro benchmark.} Since VLAs are trained on a fixed instruction distribution, they perform poorly under task perturbations. While conventional VLAs suffer catastrophic performance drops under these distributional shifts, our PRTS model exhibits superior generalization capabilities over instruction following. The best results are highlighted in bold and second best results are underlined.}
\label{tab:additional_libero_pro}
\end{threeparttable}
\end{table}